\documentclass{article}
\usepackage{amssymb,amsmath,amscd,amsfonts,amsthm,bbm,mathrsfs,yhmath, mathtools}
\usepackage{microtype}
\usepackage{graphicx}
\usepackage{subfigure}
\usepackage{booktabs}
\usepackage{hyperref}
\usepackage{cleveref}
\usepackage{physics}
\usepackage{paralist}  
\usepackage{pdfpages}

\graphicspath{{figures/}}

\usepackage{url}            
\usepackage{nicefrac}       
\usepackage{microtype}      
\usepackage{graphicx}
\usepackage{subfigure}
\usepackage{booktabs} 
\usepackage{times}		
\usepackage[T1]{fontenc}
\usepackage[utf8]{inputenc} 
\DeclareUnicodeCharacter{00A0}{~}

\usepackage[shortlabels]{enumitem}
\usepackage{caption}
\usepackage{verbatim}
\usepackage{comment}
\usepackage{tabularx}
\newcolumntype{C}{>{\centering\arraybackslash}X}

\usepackage[accepted]{icml2021}

\usepackage[textwidth=60pt, disable]{todonotes}
\usepackage{xargs}
\newcommand{\ak}[1]{\todo[color=green!40,size=\footnotesize]{ \textbf{AK:}  #1}}

\newcommand{\tb}{\textbf} 

\newcommand{\E}{{{\mathbb E}}}
\newcommand{\R}{{{\mathbb R}}}
\newcommand{\N}{\mathbb N}
\newcommand{\cH}{{{\mathcal H}}}

\newcommand{\X}{{{\mathbb{R}^d}}}

\newcommand{\cP}{{{\mathcal P}}}
\newcommand{\cK}{\mathcal K}
\newcommand{\cD}{\mathcal D}
\newcommand{\cM}{\mathcal M}
\newcommand{\p}{\partial}

\newcommand{\cF}{{{\mathcal F}}}
\newcommand{\dcF}{{\dot{\mathcal F}}}
\newcommand{\cL}{{{\mathcal L}}}
\newcommand{\kpi}{k_{\pi}}
\newcommand{\Hkpi}{\mathcal{H}_{\kpi}}
\newcommand{\Hk}{\mathcal{H}_{k}}
\newcommand{\Sop}{\mathcal{H}}

\newcommand{\ps}[1]{\langle #1 \rangle}

\DeclareMathOperator*{\argmin}{arg\,min}
\DeclareMathOperator*{\Sp}{span}
\DeclareMathOperator*{\supp}{supp}
\DeclareMathOperator*{\Lip}{Lip}
\DeclareMathOperator{\KSD}{KSD}

\DeclareMathOperator{\Hess}{Hess}
\DeclareMathOperator{\MMD}{MMD}
\DeclareMathOperator*{\divT}{div}
\renewcommand{\d}{\mathrm{d}} 

\theoremstyle{definition}
\newtheorem{theorem}{Theorem}
\newtheorem{lemma}[theorem]{Lemma}
\newtheorem{corollary}[theorem]{Corollary}
\newtheorem{proposition}[theorem]{Proposition}
\newtheorem{definition}{Definition}
\newtheorem{remark}{Remark}

\newcommand{\subscript}[2]{$#1 _ #2$}
\newlist{assumplist}{enumerate}{1}
\setlist[assumplist]{label=(\subscript{\textbf{A}}{{\arabic*}})}
\Crefname{assumplisti}{Assumption}{Assumptions}

\newlist{assumplist2}{enumerate}{1}
\setlist[assumplist2]{label=(\subscript{\textbf{B}}{{\arabic*}})}
\Crefname{assumplist2i}{Assumption}{Assumptions}

\newlist{assumplist3}{enumerate}{1}
\setlist[assumplist3]{label=(\subscript{\textbf{C}}{{\arabic*}})}
\Crefname{assumplist3i}{Assumption}{Assumptions}

\icmltitlerunning{Kernel Stein Discrepancy Descent}

\begin{document}

\twocolumn[
\icmltitle{Kernel Stein Discrepancy Descent}
\icmlsetsymbol{equal}{*}

\begin{icmlauthorlist}
\icmlauthor{Anna Korba}{ENSAE}
\icmlauthor{Pierre-Cyril Aubin-Frankowski}{Mines}
\icmlauthor{Szymon Majewski}{X}
\icmlauthor{Pierre Ablin}{ENS}
\end{icmlauthorlist}

\icmlaffiliation{Mines}{CAS, MINES ParisTech, Paris, France}
\icmlaffiliation{X}{CMAP, Ecole Polytechnique, Institut Polytechnique de Paris}
\icmlaffiliation{ENSAE}{CREST, ENSAE, Institut Polytechnique de Paris}
\icmlaffiliation{ENS}{CNRS and DMA, Ecole Normale Supérieure, Paris, France}
\icmlcorrespondingauthor{Anna Korba}{anna.korba@ensae.fr}

\icmlkeywords{Machine Learning, ICML}

\vskip 0.3in
]



\printAffiliationsAndNotice{}  

\begin{abstract}
Among dissimilarities between probability distributions, the Kernel Stein Discrepancy (KSD) has received much interest recently.
We investigate the properties of its Wasserstein gradient flow to approximate a target probability distribution $\pi$ on $\mathbb{R}^d$, known up to a normalization constant.
This leads to a straightforwardly implementable, deterministic score-based method to sample from $\pi$, named KSD Descent, which uses a set of particles to approximate $\pi$.
Remarkably, owing to a tractable loss function, KSD Descent can leverage robust parameter-free optimization schemes such as L-BFGS; this contrasts with other popular particle-based schemes such as the Stein Variational Gradient Descent algorithm.
We study the convergence properties of KSD Descent and demonstrate its practical relevance.
However, we also highlight failure cases by showing that the algorithm can get stuck in spurious local minima.
\end{abstract}

\section{Introduction} \label{sec:introduction}
An important problem in machine learning and computational statistics is to sample from an intractable target distribution $\pi$. In Bayesian inference for instance, $\pi$ corresponds to the posterior probability of the parameters, which is required to compute the posterior predictive distribution. It is known only up to an intractable normalization constant. In Generative Adversarial Networks \citep[GANs,][]{Goodfellow2014GAN}, the goal is to generate data which distribution is similar to the training set defined by samples of $\pi$. In the first setting, one has access to the score of $\pi$ (the gradient of its log density), while in the second, one has access to samples of $\pi$.

Assessing how different the target $\pi$ and a given approximation $\mu$ are can be performed through a dissimilarity function $D(\mu|\pi)$. As summarized by \citet{SimonGabriel18thesis}, classical dissimilarities include $f$-divergences such as the KL (Kullback-Leibler) or the $\chi^2$ (Chi-squared), the Wasserstein distances in Optimal Transport (OT), and Integral Probability Metrics (IPMs), such as the Maximum Mean Discrepancy \citep[MMD,][]{gretton2012kernel}. 
Dissimilarity functions can hence be used to characterize $\pi$ since, under mild assumptions, they only vanish at $\mu=\pi$. Setting $\cF(\mu)=D(\mu|\pi)$, assuming also in our case that $\pi\in\cP_2(\X)$, the set of probability measures $\mu$ with finite second moment ($\int \|x\|^2 \d\mu(x) < \infty$), the sampling task can then be recast as an optimization problem over $\cP_2(\X)$\vspace{-1mm}
\begin{equation}\label{eq:min_D}
\min_{\mu \in\cP_2(\X)} \cF(\mu).
\end{equation}
Starting from an initial distribution $\mu_0$, one can then apply a descent scheme to \eqref{eq:min_D} to converge to $\pi$. In particular, one can consider the Wasserstein gradient flow 
of $\cF$ over $\cP_2(\X)$. This operation can be interpreted as a vector field continuously displacing the particles constituting $\mu$.

Among dissimilarities, the Kernel Stein Discrepancy \citep[KSD, introduced independently by][]{liu2016kernelized,chwialkowski2016kernel,gorham2017measuring} writes as follows\vspace{-1mm}
\begin{equation}\label{eq:ksd}
    \KSD(\mu|\pi)=\sqrt{ \iint 
    \kpi(x,y)\d\mu(x)\d\mu(y)},
    \vspace{-2mm}
\end{equation}
where $\kpi$ is the Stein kernel, defined through the score of $\pi$, $s(x) = \nabla \log\pi(x)$, and through a positive semi-definite kernel $k$ (see \Cref{sec:notations} for the meaning of $\div_{1}$ or of $\nabla_{2}$)
\begin{multline}\label{eq:stein_kernel}
        \kpi(x,y)
         =s(x)^T s(y) k(x, y)+ s(x)^T \nabla_{2}k(x, y)\\
        +\nabla_{1}k(x,y)^T s(y)+\div_{1}\nabla_{2}k(x,y).
\end{multline}
The great advantage of the KSD is that it can be readily computed when one has access to the score of $\pi$ and uses a discrete measure $\hat{\mu}$, since \eqref{eq:ksd} writes as a finite double sum of $\kpi$ in this case. Furthermore the definition of the KSD was inspired by Stein's method \citep[see][for a review]{anastasiou2021steins} and no sampling over $\pi$ is required in \eqref{eq:ksd}. This motivated the use of the KSD in a growing number of problems. The KSD has been widely used in nonparametric statistical tests for goodness-of-fit \citep[e.g.][]{xu2020directional,kanagawa2020kernel}. It was also used for sampling tasks: to select a suitable set of static points to approximate $\pi$, adding a new one at each iteration \cite{chen2018stein, chen2019stein}; to compress \cite{riabiz2020optimal} or reweight \cite{hodgkinson2020reproducing} Markov Chain Monte Carlo (MCMC) outputs; and to learn a static transport map from $\mu_0$ to $\pi$ \cite{fisher2020measure}. In this paper, we consider $\cF(\mu)=\nicefrac{1}{2}\KSD^2(\mu|\pi)$ and its Wasserstein gradient $\nabla_{W_2}\cF$ to define a flow over particles to approximate $\pi$.

\tb{Related works.} Minimizing a dissimilarity $D$ is a popular approach to fit an unnormalized density model in the machine learning and computational statistics literature. For instance, \citet{hyvarinen2005estimation} proposed to minimize the Fisher divergence. Alternatively, $D$ is often taken as the KL divergence. Indeed,  since the seminal paper by \citet{Jordan1998}, the Wasserstein gradient flow of the KL has been extensively studied and related to the Langevin Monte Carlo (LMC) algorithm \citep[e.g.][]{wibisono2018sampling,durmus2019analysis}. However, an unbiased time-discretization of the KL flow is hard to implement \citep{salim2020wasserstein}. To tackle this point, a recent successful kernel-based approximation of the KL flow was introduced as the Stein Variational Gradient Descent \citep[SVGD,][]{liu2016stein}. Several variants were considered \citep[see][and references therein]{chewi2020svgd}. Another line of work considers $D$ as the MMD \cite{mroueh2019sobolev, arbel2019maximum} with either regularized or exact Wasserstein gradient flow of the MMD, especially for GAN training. However, these two approaches require samples of $\pi$ to evaluate the gradient of the MMD. As the KSD is a specific case of the MMD with the Stein kernel \cite{chen2018stein}, our approach is similar to \citet{arbel2019maximum} but better suited for a score-based sampling task, owing to the properties of the Stein kernel.

\tb{Contributions.} In this paper, in contrast with the aforementioned approaches, we choose the dissimilarity $D$ in \eqref{eq:min_D} to be the KSD. As in SVGD, our approach, KSD Descent, optimizes the positions of a finite set of particles to approximate $\pi$, but through a descent scheme of 
the KSD (in contrast to the KL for SVGD) in the space of probability measures. KSD Descent comes with several advantages. First, it benefits from a closed-form cost function which can be optimized with a fast and hyperparameter-free algorithm such as L-BFGS \cite{liu1989limited}. 
Second, our analysis comes with several theoretical guarantees (namely, existence of the flow and a descent lemma in discrete time) under a Lipschitz assumption on the gradient of $\kpi$. We also provide negative results highlighting some weaknesses of the convergence of the KSD gradient flow, such as the absence of exponential decay near equilibrium. 
Moreover, stationary points of the KSD flow may differ from the target $\pi$ and even be local minima of the flow, which implies that some particles are stuck far from high-probability regions of $\pi$. Sometimes a simple annealing strategy mitigates this convergence issue. On practical machine learning problems, the performance of KSD Descent highly depends on the local minimas of $\log \pi$. KSD Descent achieves comparable performance to SVGD on convex (i.e., $\pi$ log-concave) toy examples and Bayesian inference tasks, while it is outperformed on non-convex tasks with several saddle-points like independent component analysis.


This paper is organized as follows. \Cref{sec:background} introduces the necessary background on optimal transport and on the Kernel Stein Discrepancy. \Cref{sec:ksd_descent} presents our approach and discusses its connections with related works. \Cref{sec:convergence} is devoted to the theoretical analysis of KSD Descent. Our numerical results are to be found in \Cref{sec:experiments}.

\section{Background}\label{sec:background}

This section introduces the high-level idea of the gradient flow approach to sampling. It also summarizes the known properties of the KSD.

\subsection{Notations}\label{sec:notations}

The space of $l$ continuously differentiable functions on $\X$ is $C^{l}(\X)$. The space of smooth functions with compact support is $C_c^{\infty}(\X)$. 
If $\psi : \X \to \R^p$ is differentiable, we denote by $J \psi : \X \to \R^{p \times d}$ its Jacobian. 
If $p = 1$, we denote by $\nabla \psi$ the gradient of $\psi$.
Moreover, if $\nabla \psi$ is differentiable, the Jacobian of $\nabla \psi$ is the Hessian of $\psi$ denoted $H \psi$. If $p = d$, $\div \psi$ denotes the divergence of $\psi$, i.e.\ the trace of the Jacobian. We also denote by $\Delta \psi$ the Laplacian of $\psi$, where $\Delta \psi=\div\nabla \psi$. 
For a differentiable kernel $k: \R^{d} \times \R^{d} \rightarrow \R,$ $\nabla_{1} k$ (resp. $\nabla_{2} k$) is the gradient of the kernel w.r.t. the first (resp. second) variable, while $H_1 k$ denotes its Hessian w.r.t.\ the first variable.
 
 Consider the set $\cP_2(\X)$ of probability measures $\mu$ on $\X$ with finite second moment and  $\cP_c\left(\X\right)$ the set of probability measures with compact support. For $\mu \in \cP_2(\X)$, we denote by $\d\mu/\d\pi$ its Radon-Nikodym density if $\mu$ is absolutely continuous w.r.t.\ $\pi$.  For any $\mu \in \cP_2(\X)$, $L^2(\mu)$ is the space of functions $f : \X \to \R$ such that $\int \|f\|^2 d\mu < \infty$. 
 We denote by $\Vert \cdot \Vert_{L^2(\mu)}$ and $\ps{\cdot,\cdot}_{L^2(\mu)}$ respectively the norm and the inner product of the Hilbert space $L^2(\mu)$. 
Given a measurable map $T:\X\to \X$ and $\mu\in \cP_2(\X)$, $T_{\#}\mu$ is the pushforward measure of $\mu$ by $T$. 
We consider, for $\mu,\nu \in \cP_2(\X)$, the 2-Wasserstein distance $W_2 (\mu, \nu)$, 
and we refer to the metric space $(\cP_2(\X),W_2)$ as the Wasserstein space. In a Riemannian interpretation of $(\cP_2(\X),W_2)$, the tangent space of $\cP_2(\X)$ at $\mu$ is denoted $\mathcal{T}_{\mu}\cP_2(\X)$ and is a subset of $L^2(\mu)$ \cite{otto2001geometry}. We refer to \Cref{sec:W2_diff} for more details on the Wasserstein distance and related flows.


\subsection{Lyapunov analysis and gradient flows}
\label{sec:sub_Lyap}

To sample from a target distribution $\pi$, a now classical approach consists in identifying a continuous process which moves particles from an initial probability distribution $\mu_0$ toward samples of $\pi$. This can be expressed as searching for vector fields $v_t:\R^d\rightarrow \R^d$ transporting the distribution $\mu_t$ through the continuity equation (see \Cref{sec:continuity_eq})
\begin{equation}\label{eq:continuity_eq}
\frac{\partial \mu_t}{\partial t}+\divT(\mu_t v_t ) =0
\end{equation}
where $v_t$ should ensure the convergence of $\mu_t$ to $\pi$, for some topology over measures, in finite or infinite time.  Due to $v_t$, \cref{eq:continuity_eq} is nonlinear over $\mu_t$. Cauchy-Lipschitz-style assumptions for existence and uniqueness of the solution of \eqref{eq:continuity_eq} are provided in \Cref{sec:Cauchy-Lip_flow}. The continuity equation ensures that the mass is conserved and that it is not teleported as for a mixture $\mu_t=(1-t)\mu_0+t\pi$. In order to adjust the position of particles only depending on the present distribution $\mu_t$ and to have an automated choice of $v_t$ at any given time, it is favorable to have $v_t$ as a function of $\mu_t$, written as $v_{\mu_t}$.

A principled way to select such a $v_{\mu_t}$ is to define it based on a Lyapunov functional $\cF(\mu)$ over measures, decreasing along the Wasserstein gradient flow (see \Cref{sec:W2_diff})
\begin{equation}\label{eq:Lyap_deriv}
 \dcF(\mu_t):=\frac{\d \cF(\mu_t)}{\d t}=\ps{\nabla_{W_2}\cF(\mu_t),v_{\mu_t}}_{L^2(\mu_t)}\le 0.
\end{equation}
Any dissimilarity $\cF(\cdot)=D(\cdot|\pi)$ is a valid Lyapunov candidate since it is non-negative and vanishes at $\pi$. Hence, \eqref{eq:continuity_eq} can be seen as a continuous descent scheme of \eqref{eq:min_D} or, conversely, \eqref{eq:min_D}-\eqref{eq:Lyap_deriv} can be interpreted as a way to choose $v_{\mu_t}$ in \eqref{eq:continuity_eq} to steer $\mu_0$ to $\pi$. In short, any Lyapunov-based approach rests upon three quantities ($\cF(\mu_t),v_{\mu_t},\dcF(\mu_t)$), related by \eqref{eq:Lyap_deriv}. A natural choice of $v_{\mu_t}$ satisfying \eqref{eq:Lyap_deriv} and realizing the steepest descent is the Wasserstein gradient itself, $v_{\mu_t}=-\nabla_{W_2}\cF(\mu_t)$. Depending on the choice of $\cF$, this $v_{\mu_t}$ may be hard to implement, or require specific analysis of the resulting dissipation function $\dcF(\mu_t)$. Otherwise, to ensure that $\dcF(\mu_t)$ only vanishes at $\pi$, one can choose $v_{\mu_t}$ so that both $\cF(\mu_t)$ and $\dcF(\mu_t)$ are known dissimilarities. As a matter of fact, if there exists a dissimilarity $\tilde{D}$ separating measures such that $-\dcF(\mu)\ge \tilde{D}(\mu|\pi)$, then $\pi$ is asymptotically stable for the flow and, if $\tilde{D}(\mu|\pi)\ge \cF(\mu)$, then $\pi$ is exponentially stable (by Gronwall's lemma). This relates the Lyapunov analysis to functional inequalities \cite{villani2003optimal}, expressing domination w.r.t.\ $\cF$ of the Wasserstein gradient of $\cF$ under specific assumptions on $\pi$, e.g.\ log-Sobolev for the KL, or Poincaré for the $\chi^2$ \cite{chewi2020svgd}.

\subsection{Kernel Stein Discrepancy}\label{sec:background_ksd}

Consider a positive semi-definite kernel $k : \R^d \times \R^d \to \R$ and its corresponding RKHS $\cH_k$ of real-valued functions on 
$\R^d$. The space $\mathcal{H}_k$ is a Hilbert space with inner product $\ps{\cdot,\cdot}_{\cH_k}$ and norm $\Vert \cdot \Vert_{\cH_k}$. Moreover, $k$ satisfies the reproducing property: $
\forall \; f \in \cH_k,\; f(x)=\ps{f,k(x,\cdot)}_{\cH_k}$; which for smooth kernels also holds for derivatives, e.g.\ $\p_i f(x)=\ps{f,(\nabla_1 k(x,\cdot))_i}_{\cH_k}$ \citep[see][]{saitoh16theory}. Let $\mu \in \cP_2(\X)$. If $\int k(x,x) \d\mu(x)<\infty$, then the integral operator associated to the kernel $k$ and measure $\mu$, denoted by
$S_{\mu, k}: L^2(\mu) \rightarrow \Hk$ and defined as\vspace{-3mm}
\begin{equation}\label{eq:integral_operator}
S_{\mu, k}f  = \int k(x,\cdot)f(x)\d\mu(x),
\vspace{-3mm}
\end{equation}
is a Hilbert-Schmidt operator and $\Hk\subset L^2(\mu)$. In this case, the identity embedding $\iota:\Hk\to L^2(\mu)$ is a bounded operator and it is the adjoint of $S_{\mu,k}$ (i.e., $\iota^*=S_{\mu,k}$ \citep[Theorems 4.26 and 4.27]{steinwart2008support}. Hence, for any $(f,g)\in \Hk\times L^2(\mu)$, $\ps{\iota f,g}_{L^2(\mu)}=\ps{f,S_{\mu,k}g}_{\Hk}$.
We denote by $\cH_k^d$ the Cartesian product RKHS consisting of elements $f=(f_1, \dots, f_d)$ with $f_i\in \cH_k$, and with inner product $\ps{f,g}_{\cH_k^d}=\sum_{i=1}^d\ps{f_i,g_i}_{\cH_k}$.\ For vector-inputs, we extend $S_{\mu, k}$, applying it component-wise.

The Stein kernel $\kpi$ \eqref{eq:stein_kernel} is a reproducing kernel and satisfies a Stein identity ($\int_{\X} \kpi(x,\cdot)\d\pi(x)=0$) under mild regularity assumptions on $k$ and $\pi$.\footnote{e.g., $k$ is a Gaussian kernel and $\pi$ is a smooth density fully supported on $\R^d$, see \citet[Theorem 3.7]{liu2016kernelized}.} It allows for several interpretations of the KSD \eqref{eq:ksd} as already discussed by \citet{liu2016kernelized}. It can be introduced as an IPM in the specific case of a Stein operator applied to $\Hk$ \citep[e.g.][]{gorham2017measuring}. It can then be identified as an asymmetric MMD in $\Hkpi$ (see \Cref{sec:related_work}). 
Alternatively, the squared KSD can be seen as a kernelized Fisher divergence, where the Fisher information $\| \nabla \log \left(\frac{\d\mu}{\d\pi}\right)\|^2_{L^2(\mu)}$ is smoothed through the kernel integral operator, i.e.\ $\KSD^2(\mu|\pi)= \|S_{\mu, k}\nabla \log\left(\frac{\d\mu}{\d\pi}\right)\|^2_{\cH^d_k}$. In this sense, the squared KSD has also been referred to as the Stein Fisher information \cite{duncan2019geometry}. Hence, minimizing the KSD can be thought as a kernelized version of score-matching \cite{hyvarinen2005estimation}.

The metrization of weak convergence by the KSD, i.e.\ that 
$\lim_{t\to \infty}\KSD(\mu_t|\pi) =0$ 
implies the weak convergence of $\mu_t$ to $\pi$, depends on the choice of the kernel relatively to the target. This question has been considered by~\citet{gorham2017measuring}, who show this is the case under assumptions akin to strong log-concavity of $\pi$ at infinity \citep[namely distant dissipativity,][]{eberle2016reflection}, and for a kernel $k$ with a slow decay rate. 
This includes finite Gaussian mixtures with common covariance and 
kernels that are translation-invariant with heavy-tails and non-vanishing Fourier transform, such as the inverse multi-quadratic (IMQ) kernel defined by $k(x,y)=(c^2 + \|x- y\|_2^2)^{\beta}$ for $c >0$ and $\beta \in (-1, 0)$, or its variants considered in \citet{chen2018stein}. 
	 


\section{Sampling as optimization of the KSD}\label{sec:ksd_descent}


This section defines the KSD Descent and relates it to other gradient flows, especially the MMD gradient flow, of which the KSD Descent is a special case. In all the following, we assume that $k\in C^{3,3}(\R^d\times \R^d, \R)$, and that $\pi$ is such that $s=\nabla\log\pi\in C^2(\R^d)$.

\subsection{Continuous time dynamics}

Consider the functional
$\cF : \cP_2(\X) \to [0,+\infty)$, $\mu \mapsto \frac{1}{2}\KSD^2(\mu|\pi)$  
defined over the Wasserstein space. 
If $\mu \in \cP_2(\X)$ satisfies some mild regularity conditions (i.e., it has a $C^1$ density w.r.t. Lebesgue measure, and it is in the domain of $\cF$, see \Cref{sec:W2_diff}), the  gradient of $\cF$ at $\mu$ is well-defined and denoted by $\nabla_{W_2} \cF(\mu) \in L^2(\mu)$. We shall consider the following assumptions on the Stein kernel:
\vspace{-0.15cm}
\newcounter{contlist}
\begin{assumplist}
	\setlength\itemsep{0.2em}
	\item \label{ass:lipschitz} There exists a map $L(\cdot)\in C^0(\X,\R_+)$, which is $\mu$-integrable for any $\mu \in \cP_2(\X)$, such that, for any $y \in \X$, the maps $x\mapsto \nabla_1 \kpi(x,y)$ and $x\mapsto \nabla_2 \kpi(x,y)$ are $L(y)$-Lipschitz. 
	\item \label{ass:growth} There exists $m>0$ such that for any $\mu \in \cP_c\left(\R^{d}\right),$ for all $y \in \R^{d},$ we have $\|\int \nabla_2 k_{\pi}(x,y) \d\mu(x)\| \leq m\left(1+\|y\|+\int \|x\|\,\d\mu(x)\right).$
	\item \label{ass:square_der_int} The map $(x,y) \mapsto \|H_1 \kpi(x,y)\|_{op}$ is $\mu\otimes\nu$-integrable for every $\mu,\nu\in \cP_2(\X)$.
	\item \label{ass:integrable} For all $\mu\in \cP_2(\X)$, $\int \kpi(x,x)\d\mu(x)<\infty$.
\end{assumplist}
\vspace{-0.15cm}
The KSD gradient flow is defined as the flow induced by the continuity equation:
\begin{equation}\label{eq:ksd_flow}
\frac{\partial \mu_t}{\partial t}+\divT(\mu_t v_{\mu_t} ) =0 \text{ for } v_{\mu_t} := - 
\nabla_{W_2}\cF(\mu_t).
\end{equation}
\Cref{ass:lipschitz,ass:growth} ensure that the KSD gradient flow exists and is unique, they are further discussed in \Cref{sec:Cauchy-Lip_flow}. \Cref{ass:lipschitz,ass:square_der_int} are needed so that the Hessian of $\cF$ is well defined (see \Cref{sec:convergence}). \Cref{ass:integrable} guarantees that the integral operator $S_{\mu, \kpi}$ \eqref{eq:integral_operator} is well-defined and that $\cF(\mu)<\infty$ for all $\mu\in \cP_2(\X)$.
\begin{lemma}\label{lem:lipschitz_ass}
	 Assume that $k$, its derivatives up to order 3, and their product by $\|x-y\|$ are uniformly bounded over $\X$; and that $s$ is Lipschitz and has a bounded Hessian over $\X$. Then \Cref{ass:lipschitz,ass:square_der_int,ass:integrable} hold. If, furthermore there exists $M>0$ and $M_0$ such that, for all $x\in\X$, $\|s(x) \|\le M\sqrt{\|x\|}+M_0$, then \Cref{ass:growth} also holds.
\end{lemma}
See the proof in \Cref{sec:proof_lem_lipschitz_ass}. Smoothed Laplace distributions $\pi$ paired with Gaussian $k$ satisfy the assumptions of \Cref{lem:lipschitz_ass}. For Gaussian $\pi$, $s$ is linear, so \Cref{ass:lipschitz,ass:square_der_int,ass:integrable} hold for smooth kernels, but \Cref{ass:growth} does not hold in general because of the $s(x)^\top s(y)$ term in $\kpi$. Notice that most of our results are stated without \Cref{ass:growth}, which is only required to establish the global existence of KSD flow, in the sense that the particle trajectories are well-defined and do not explode in finite-time.
 
\begin{proposition}\label{prop:dissipation}
Under \Cref{ass:lipschitz,ass:growth}, the $W_2$ gradient of $\cF$ evaluated at $\mu$ and its dissipation \eqref{eq:Lyap_deriv} along \eqref{eq:ksd_flow} are
\begin{align}
\nabla_{W_2}\cF(\mu)&=\mathbb{E}_{x \sim \mu}[\nabla_{2}\kpi(x,\cdot)],\label{eq:ksd_w2_grad}\\
\dot{\cF}(\mu_t)&=-\E_{y\sim \mu_t}\left[ \| \mathbb{E}_{x \sim \mu_t}[\nabla_{2}\kpi(x,y)] \|^2\right] \label{eq:dissipation}.
\end{align}
\end{proposition}\vspace{-3mm}
Since the r.h.s. of \eqref{eq:dissipation} is negative, \Cref{prop:dissipation} shows that the squared KSD w.r.t.\ $\pi$ decreases along the KSD gradient flow dynamics. In other words, $\cF$ is indeed a Lyapunov functional for the dynamics \eqref{eq:ksd_flow} as discussed in \Cref{sec:sub_Lyap}.

%
\subsection{Discrete time and discrete measures}

A straightforward time-discretization of \eqref{eq:ksd_flow} is a gradient descent in the Wasserstein space applied to ${\cF(\mu)=\frac{1}{2}\KSD^2(\mu|\pi)}$. Starting from an initial distribution $\mu_0 \in \cP_2(\X)$, it writes as follows at iteration $n\in \N$,
\begin{equation}\label{eq:ksd_descent}
    \mu_{n+1} = \left(I -\gamma \nabla_{W_2} \cF(\mu_n)\right)_{\#} \mu_n,
\end{equation}
for a step-size $\gamma>0$. However for discrete measures $\hat{\mu}=\frac{1}{N}\sum_{j=1}^{N}\delta_{x^{j}}$, we can make the problem more explicit setting a loss function
\begin{equation}
    F([x^j]_{j=1}^{N}) :=
    \cF(\hat{\mu})
    =\frac1{2N^2}\sum_{i, j=1}^N\kpi(x^i, x^j).
    \label{eq:discrete_loss}
\end{equation}
Problem \eqref{eq:min_D} then corresponds to a standard non-convex optimization problem over the finite-dimensional, Euclidean, space of particle positions. The gradient of $F$ is readily obtained as\vspace{-2mm}
\begin{equation*}
     \nabla_{x^i} F([x^j]_{j=1}^{N}) = \frac{1}{N^2}\sum_{j=1}^N \nabla_{2}\kpi( x^{j},x^{i}).
\end{equation*}
since, by symmetry of $\kpi$, $\nabla_1\kpi(x,y)=\nabla_2\kpi(y,x)$. As both $F$ and $\nabla_{x^i} F$ can be explicitly computed, one can implement the KSD Descent either using a gradient descent (\Cref{alg:ksd_descent_GD}) or through a quasi-Newton algorithm such as L-BFGS (\Cref{alg:ksd_descent_LBFGS}). As a matter of fact, L-BFGS \cite{liu1989limited} is  often faster and more robust than the conventional gradient descent. It also does not require choosing critical hyper-parameters, such as a learning rate, since L-BFGS performs a line-search to find suitable step-sizes. It only requires a tolerance parameter on the norm of the gradient, which is in practice set to machine precision.
\begin{algorithm}[ht]
   \caption{KSD Descent GD}
   \label{alg:ksd_descent_GD}
\begin{algorithmic}
   \STATE {\bfseries Input:} initial particles $(x_0^i)_{i=1}^{N}\sim \mu_0$, number of iterations $M$, step-size $\gamma$
   \FOR{$n=1$ {\bfseries to} $M$}
   \STATE
      \vspace{-0.8cm}
{\begin{equation}\label{eq:ksd_update}
    \hspace{-0.5cm}[x_{n+1}^{i}]_{i=1}^{N}=[x_n^{i}]_{i=1}^{N} - \frac{\gamma}{N^2}\sum_{j=1}^N [\nabla_{2}\kpi( x_n^{j},x_n^{i})]_{i=1}^{N},
    \end{equation}}
   \vspace{-0.8cm}
   \ENDFOR
   \STATE \textbf{Return: }$[x_{M}^i]_{i=1}^{N}$.
   \end{algorithmic}
\end{algorithm}
\begin{algorithm}[ht]
   \caption{KSD Descent L-BFGS}
   \label{alg:ksd_descent_LBFGS}
\begin{algorithmic}
   \STATE {\bfseries Input:} initial particles $(x_0^i)_{i=1}^{N}\sim \mu_0$, tolerance $\mathrm{tol}$
   \vspace{0.2cm}
   \STATE \textbf{Return: }$[x_*^i]_{i=1}^{N} = \mathrm{L}\text{-}\mathrm{BFGS}(F, \nabla F, [x_0^i]_{i=1}^{N}, \mathrm{tol})$.
    \end{algorithmic}
\end{algorithm}
A technical descent lemma for \eqref{eq:ksd_descent} (\Cref{prop:descent_lemma_appendix}) showing that $\cF$ decreases at each iteration \eqref{eq:ksd_descent} is provided in \Cref{sec:descent_lemma}. It requires the boundedness of $(\|L(\cdot)\|_{L^2(\mu_n)})_{n\ge0}$, the $L^2$-norm of the Lipschitz constants of \Cref{ass:lipschitz} along the flow, as well as the convexity of $L(\cdot)$ and a compactly-supported initialization.

\begin{remark}
As L-BFGS requires access to exact gradients, \Cref{alg:ksd_descent_LBFGS} cannot be used in a stochastic setting. However this can be done for \Cref{alg:ksd_descent_GD} by subsampling over particles in the double sum in \Cref{eq:discrete_loss}. Moreover, in some settings like Bayesian inference, the score itself writes as a sum over observations. In this case, the loss $F$ writes as a \emph{double} sum over observations, and a stochastic variant of \Cref{alg:ksd_descent_GD} tailored for this problem could be devised, in the spirit of~\citet{clemenccon2016scaling}. 
\end{remark}

\subsection{Related work}\label{sec:related_work}
Several recent works fall within the framework sketched in Section \ref{sec:sub_Lyap}. In SVGD, \citet{liu2016kernelized} take $\cF$ as the KL, and set $v_{\mu_t}=-S_{\mu, k}\nabla\ln\left(\frac{d\mu_t}{d\pi}\right)$ to obtain $\dcF$ as the (squared) KSD. Integrating this inequality w.r.t. time yields a $1/T$ convergence rate for the average $\KSD$ between $\mu_t$ and $\pi$ for $t\in [0,T]$. This enabled \citet[Proposition 5, Corollary 6]{korba2020non} to obtain a discrete-time descent lemma for bounded kernels, as well as rates of convergence for   
the averaged $\KSD$. In contrast, since the dissipation \eqref{eq:dissipation} of the $\KSD$ along its $W_2$ gradient flow does not correspond to any dissimilarity, our descent lemma for \eqref{eq:ksd_descent} (\Cref{prop:descent_lemma_appendix}) does not yield similar rates of convergence.
Alternatively, in the LAWGD algorithm recently proposed by \citet{chewi2020svgd}, $\cF$ is the KL, and $v_{\mu_t}=-\nabla S_{\pi, k_{\cL_\pi}}\left(\frac{d\mu_t}{d\pi}\right)$ with $k_{\cL_\pi}$ chosen such that $\dcF$ is the $\chi^2$, by taking $S_{\pi, k_{\cL_\pi}}$ as the inverse of the diffusion operator:
\begin{equation}\label{eq:def_diffusion_op}
    \cL_\pi : f\mapsto-\Delta f-\langle \nabla \log\pi,\nabla f \rangle.
\end{equation}
Their elegant approach results in a linear convergence of the KL along their flow, but implementing LAWGD in practice requires to compute the spectrum of $\cL_\pi$. It is in general as difficult as solving a linear PDE, and \citet{chewi2020svgd} admit it is unlikely to scale in high dimensions.\footnote{The update rules of SVGD, LAWGD and MMD-GD can be found in \Cref{sec:updates_related_work}.}

Beyond studies on the KL, \citet{mroueh2019sobolev} considered $\cF$ as the MMD and pick $v_{\mu_t}$ based on a kernelized Sobolev norm so that $\dcF$ resembles the MMD, but without proving convergence of their scheme. \citet{arbel2019maximum} also analyzed $\cF$ as the MMD, but for $v_{\mu_t}=-\nabla_{W_2}\frac{1}{2}\MMD^2(\mu_t,\pi)$ with similar $\dcF$ as ours and with a dedicated analysis of their MMD-GD flow. We recall that $\MMD^2(\mu,\pi)=\|\int k(x,\cdot)d\mu(x)-\int k(x,\cdot)d\pi(x)\|_{\cH_k}^2$.
Since the Stein kernel satisfies the Stein's identity $\int \kpi(x,\cdot)d\pi(x)=0$, the KSD  \eqref{eq:ksd} can be identified to an MMD with the Stein kernel \cite{chen2018stein}. However, the assumptions of \citet{arbel2019maximum} -$\nabla k$ is $L$-Lipschitz for $L\in \R^+$- do not hold in general for unbounded Stein kernels. Here, we provide the right set of assumptions~\ref{ass:lipschitz}-\ref{ass:growth} on $\kpi$ for the flow to exist and for a descent lemma to hold. Also, as noted on \Cref{fig:toy_example}, the sample-based MMD flow, defined through
\begin{multline}\label{eq:mmd_grad}
    \nabla_{W_2}\frac{1}{2}\MMD^2(\mu,\pi)= \int \nabla_{2}k(x,\cdot) \d(\mu-\pi)(x)
\end{multline}
can fail dramatically while KSD flow succeeds. 
This suggests that the geometrical properties of the KSD flow are more favorable than the ones of the regular MMD flow. In other words, choosing an appropriate (target-dependent) kernel, as in our method or in LAWGD, appears more propitious than 
taking a 
kernel $k$ unrelated to $\pi$.

Related to the optimization of the $\KSD$, Stein points~\citep{chen2018stein} also propose to use the KSD loss for sampling, but the loss is minimized using very different tools. While KSD descent uses a first order information by following the gradient (or L-BFGS) direction with a fixed number of particles, Stein points use a Frank-Wolfe scheme, adding particles one by one in a greedy fashion. Given $N$ particles $x^1, \dots, x^N$, in Stein points, the next particle is set as
$$
x^{N+1} \in \argmin_{x} \frac12k_{\pi}(x, x) + \sum_{i=1}^N k_{\pi}(x, x^i).
$$
This problem is solved using derivative-free (a.k.a. zeroth-order) algorithms like grid-search or random sampling. The main drawback of such an approach is that it scales poorly with the dimension $d$ when compared to first-order methods. In the same spirit, \cite{futami2019bayesian} have recently proposed a similar algorithm to optimize the MMD.

\section{Theoretical properties of the KSD flow}\label{sec:convergence}

In this section, we provide a theoretical study of the convergence of the KSD Wasserstein gradient flow, assessing the convexity of $\cF$ and discussing the stationary points of its gradient flow. Remarkably, we encounter pitfalls similar to other deterministic flows derived from IPMs. This issue arises because IPMs are always mixture convex, but seldom geodesically convex. We first investigate the convexity properties of $\cF$ along $W_2$ geodesics and show that exponential convergence near equilibrium cannot hold in general. Then, we examine some stationary points of its $W_2$ gradient flow, which explain the failure cases met in \Cref{sec:experiments}, where $\hat{\mu}_n$ converges to a degenerate measure.


\subsection{Convexity properties of the KSD flow}

As is well-known, decay along $W_2$ gradient flows can be obtained from convexity properties along geodesics. A natural object of interest is then the Hessian of the objective $\cF$. We define below this object, in a similar way as \citet{duncan2019geometry}. We recall that  $\{\nabla \psi,\enspace \psi\in C_c^{\infty}(\X)\}$ is by definition dense in $\mathcal{T}_{\mu}\cP_2(\R^d)\subset L^2(\mu)$ for any $\mu\in \cP_2(\X)$ \citep[Definition 8.4.1]{ambrosio2008gradient}.
\begin{definition}\label{def_Hessian}
Consider $\psi\in C_c^{\infty}(\X)$ and the path $\rho_t$ from  $\mu$ to $(I+\nabla\psi)_{\#}\mu$ given by: $\rho_t=  (I+t\nabla\psi)_{\#}\mu$, for all $t\in [0,1]$.
The Hessian of $\cF$ at $\mu$, $H\cF_{|\mu}$, is defined as a symmetric bilinear form on $C_c^{\infty}(\X)$ associated with the quadratic form
$\Hess_{\mu}\cF(\psi,\psi) := \frac{d^2}{dt^2}\Bigr|_{\substack{t=0}}\mathcal{F}(\rho_t).$
\end{definition}
\Cref{def_Hessian} can be straightforwardly related to the usual symmetric bilinear form defined on $\mathcal{T}_{\mu}\cP_2(\R^d)\times \mathcal{T}_{\mu}\cP_2(\R^d)$ \citep[Section 3]{otto2000generalization}\footnote{The $W_2$ Hessian of $\cF$, denoted $H\cF_{|\mu}$ is an operator over $\mathcal{T}_{\mu}\mathcal{P}_2(\X)$ verifying $\ps{H\cF_{|\mu}v_t, v_t}_{L^2(\mu)}=\frac{d^2}{dt^2}\Bigr|_{\substack{t=0}}\mathcal{F}(\rho_t)$ if $t\mapsto \rho_t$ is a geodesic starting at $\mu$ with vector field $t\mapsto v_t$.}.
\begin{proposition}\label{prop:hessian_ksd_mu}
Under \Cref{ass:lipschitz,ass:square_der_int}, 
the Hessian of $\cF$ at $\mu$ is given, for any $\psi \in C_c^{\infty}(\R^d)$, by
\begin{multline}\label{eq:hessian_ksd_mu}
\hspace{-0.4cm}   \Hess_{\mu}\cF(\psi,\psi) = \E_{x,y\sim \mu}\left[\nabla\psi(x)^T\nabla_{1}\nabla_2 \kpi(x,y)\nabla\psi(y) \right] \\
    + \E_{x,y\sim \mu}\left[\nabla\psi(x)^T H_{1}\kpi(x,y)\nabla\psi(x) \right].
\end{multline}
\end{proposition}\vspace{-3mm}
A proof of \Cref{prop:hessian_ksd_mu} is provided in \Cref{sec:proof_hessian_ksd_mu}. Our computations are similar to the ones in \citep[Lemma 23]{arbel2019maximum} with some terms getting simpler owing to the Stein's identity satisfied by the Stein kernel. As for the squared MMD \citep[Proposition 5]{arbel2019maximum}, the squared KSD is unlikely to be geodesically convex. Indeed, while the first term is always positive, the second term in \eqref{eq:hessian_ksd_mu} can in general take negative values, unless $H_1 \kpi(x,y)$ is positive for all $x,y\in \supp(\mu)$. 
Nevertheless, at $\mu=\pi$, this second term vanishes, again owing to the Stein's property of $\kpi$.

\begin{corollary}\label{prop:hessian_ksd_pi}
Under \Cref{ass:lipschitz,ass:square_der_int,ass:integrable}, the Hessian of $\cF$ at $\pi$ is given, for any $\psi \in C_c^{\infty}(\R^d)$, by
\begin{equation*}
    \Hess_{\pi}\cF(\psi,\psi) = \| S_{\pi,\kpi}\cL_{\pi} \psi\|^2_{\Hkpi}
\end{equation*}
where $S_{\pi,\kpi}$ and $\cL_{\pi}$ are defined in \Cref{eq:integral_operator,eq:def_diffusion_op}.
\end{corollary}
A proof of \Cref{prop:hessian_ksd_pi} is provided in \Cref{sec:proof_hessian_ksd_pi}. We now study the curvature properties near equilibrium, characterized by $\Hess_{\pi}\cF$. In particular, inspired by the methodology described in \citet{villani2003optimal} and recently applied by \citet{duncan2019geometry}, we expect exponential convergence of solutions initialized near $\pi$ whenever the Hessian is bounded from below by a quadratic form on the tangent space of $\cP_2(\X)$ at $\pi$, included in $L^2(\pi)$.
\begin{definition}\label{def:exp_decay}
We say that \textit{exponential decay near equilibrium} holds if there exists $\lambda>0$ such that for any $\psi \in C_c^{\infty}(\R^d)$,
\begin{equation}\label{eq:exp_decay}
\Hess_{\pi}\cF(\psi,\psi) \ge \lambda  \|\nabla \psi\|^2_{L^2(\pi)}.
\end{equation}
\end{definition}\vspace{-3mm}
According to \Cref{prop:hessian_ksd_pi}, \eqref{eq:exp_decay} can be seen as a kernelized version of the following form of the Poincaré inequality for $\pi$ \citep[Chapter 5]{bakry2013analysis}
\begin{equation} \label{eq:poincare_dual_form}
    \Vert \cL_{\pi} \psi \Vert_{L_2(\pi)}^2 \geq \lambda_{\pi} \Vert \nabla \psi \Vert_{L_2(\pi)}^2.
\end{equation}
Condition \eqref{eq:exp_decay} is similar to the Stein-Poincaré inequality \citep[Lemma 32]{duncan2019geometry}. We will now argue that \eqref{eq:exp_decay} is hardly ever satisfied, thus obtaining an impossibility result reminiscent of the one for SVGD in \citep[Lemma 36]{duncan2019geometry}, which states that exponential convergence (of the KL) for the SVGD gradient flow does not hold whenever $\pi$ has exponential tails and the derivatives of $\nabla \log \pi$ and $k$ grow at most at a polynomial rate. We start with the following characterization of exponential decay near equilibrium:

\begin{proposition}\label{prop:no_exp_cv}  Let $T_{\pi,\kpi}=S_{\pi,\kpi}^* \circ S_{\pi,\kpi}$ and $L_0^2(\pi)=\{\phi \in L^2(\pi), \int \phi d\pi = 0\}$. The exponential decay near equilibrium \eqref{eq:exp_decay} holds if and only if $\cL_{\pi}^{-1} : L_0^2(\pi) \rightarrow L_0^2(\pi)$, the inverse of $\cL_{\pi}\vert_{L_0^2(\pi)}$, is well-defined, bounded, and for all $\phi \in  L_0^2(\pi)$ we have
 \begin{equation}\label{eq:proof_prop_no_exp_cv_reformulation_main}
     \langle \phi , T_{\pi, \kpi} \phi \rangle_{L_2(\pi)} \geq \lambda \langle \phi , \cL_{\pi}^{-1} \phi \rangle_{L_2(\pi)}.
 \end{equation}
 \end{proposition}\vspace{-3mm}
See \Cref{sec:no_exp_cv} for a proof. By the spectral theorem for compact, self-adjoint operators \citep[Section 8.3]{kreyszig1978introductory}, $T_{\pi,\kpi}$ has a discrete spectrum $(l_n)_{n\in \mathbb{N}^*}$ which satisfies $l_n\ge0$ and $l_n\to 0$. Under mild assumptions on $\pi$, the operator $\cL_{\pi}$ also has a discrete, positive spectrum \citep[Appendix A]{chewi2020svgd}. \Cref{prop:no_exp_cv} implies the following necessary condition on the spectrum of $\cL_{\pi}^{-1}$ and $T_{\pi,\kpi}$ for the exponential decay near equilibrium \eqref{eq:exp_decay} to hold.

\begin{corollary} \label{coro:spectr_decay}\
If $\cL_{\pi}^{-1}$ has a discrete spectrum $(\lambda_n)_{n\in \mathbb{N}^*}$ and \eqref{eq:exp_decay} 
holds, then
   $  \lambda_n=\mathcal{O}(l_n)$,
 i.e.\ the eigenvalue decay of $\cL_{\pi}^{-1}$ is at least as fast as the one of $T_{\pi,\kpi}$.
\end{corollary}

We also show that if $\Hkpi$ is infinite dimensional and exponential convergence near equilibrium holds, then $\cL_{\pi}^{-1}$ has a discrete spectrum (\Cref{lem:discrete_spectrum_hkpi}). 
We now present our impossibility result on the exponential decay near equilibrium.

\begin{theorem}\label{cor:no_exp_decay}
Let $\pi \propto e^{-V}$. Assume that $V \in C^2(\R^d)$, $\nabla V$ is Lipschitz and $\cL_{\pi}$ has a discrete spectrum. Then exponential decay near equilibium does not hold.
\end{theorem} 

The main idea behind the proof of \Cref{cor:no_exp_decay}  (\Cref{sec:eigenvalue_decay_proof}) is that $T_{\pi,\kpi}$ is nuclear \citep[Theorem 4.27]{steinwart2008support}, which implies that its eigenvalues $(l_n)_{n\in \mathbb{N}^*}$ are summable. On the other hand the eigenvalue decay of $\cL_{\pi}^{-1}$ when $\pi$ is a Gaussian can be seen to be of order $O(1/n^{1/d})$ (\Cref{lem:eigenvalue_decay} in \Cref{sec:eigenvalue_decay_proof}), which is not summable. The general case is obtained by comparison with a Gaussian.
Despite the lack of (strong) geodesic convexity near equilibrium, we still observe empirically good convergence properties of the KSD flow for discrete measures to a stationary measure. Hence, we now investigate these stationary measures.

\subsection{Stationary measures of the KSD flow}\label{sec:spurious_minima}

The KSD gradient flow leads to a deterministic algorithm, as for SVGD and LAWGD. To study the convergence of these algorithms in continuous time, it is relevant to characterize the stationary measures, i.e.\ the ones which cancel the dissipation $\dcF$ \eqref{eq:Lyap_deriv} of the objective functional $\cF$ along the relative gradient flow dynamics. Unfortunately, unlike for the SVGD and LAWGD algorithms, the dissipation related to the KSD flow \eqref{eq:dissipation} does not yield a dissimilarity between measures. 
Consequently, the study of the stationary measures of the KSD is more involved. We discuss below when failure cases may happen.

\begin{lemma}\label{lem:no_cst}
Assume \Cref{ass:integrable} holds. Then $\Hkpi$ does not contain non-zero constant functions.
\end{lemma}
A proof of \Cref{lem:no_cst} is provided in \Cref{sec:proof_no_cst}. This result has the immediate consequence of considerably restricting the number of candidate fully-supported measures that are stationary for the KSD gradient flow. Consider one such measure $\mu_{\infty}$. At equilibrium, $\dot{\cF}(\mu_{\infty})=0$; which implies that $\int \kpi(x,.)d\mu_{\infty}(x)$ is $\mu_{\infty}$-a.e. equal to a constant function $c$. Since $x\mapsto c$ is then also an element of $\Hkpi$, the previous lemma implies that $c=0$. Hence, if $\mu_{\infty}$ and $\pi$ are full-support, $\cF(\mu_\infty)=0$. Provided $\kpi$ is characteristic \citep{bharath2011charac}, then $\mu_{\infty}=\pi$.

However, as  \Cref{alg:ksd_descent_GD,alg:ksd_descent_LBFGS} rely on discrete measures, the dissipation $\dcF$ \eqref{eq:dissipation} can vanish even for $\mu\neq \pi$ because $\mu$ is not full-support. Depending on the properties of $\pi$ and $k$, this may happen even for trivial measures such as a single Dirac mass, as stated in the following Lemma.
\begin{lemma}\label{cor:stable Dirac}
Let $x_0$ such that $s(x_0)=0$ and $Js(x_0)$ is invertible, and consider a translation-invariant kernel $k(x,y)=\phi(x-y)$, for $\phi\in C^3(\X)$. Then $\delta_{x_0}$ is a stable fixed measure  of \eqref{eq:ksd_flow}, i.e. it is stationary
and any small push-forward of $\delta_{x_0}$ is attracted back by the flow.
\end{lemma}
\emph{Proof:} For $\varepsilon > 0$ and $\psi\in C^\infty_c(\R^d)$, set $\mu_{\varepsilon} = (I + \varepsilon \psi)_\#\delta_{x_0} $. We then have $\cF(\mu_{\varepsilon}) =\frac{1}{2} \kpi(x_0 + \varepsilon \psi(x_0), x_0 + \varepsilon \psi(x_0))$.
Expanding $\kpi(x, x)$ at the first order around $x = x_0$ gives $2\cF(\mu_{\varepsilon}) = \varepsilon^2 \|[Js(x_0)]\psi(x_0)\|^2 \phi(0)-\Delta \phi(0) + o(\varepsilon^2)$.
This quantity is minimized for $\psi(x_0) = 0$, which shows that $\delta_{x_0}$ is indeed a local minimum for $\cF$.

Importantly, this result applies whenever the score $s$ vanishes at $x_0$, not only when $x_0$ is a local stable minimum of the potential $\log(\pi)$. This means that for a single particle, KSD descent is attracted to any stationary point of $\log(\pi)$, whereas SVGD converges only to local maxima of $\log(\pi)$ \cite{liu2016stein}. Nonetheless, if $\pi$ is log-concave, there is no spurious stationary point.

For cases more general than \Cref{cor:stable Dirac}, we are interested in the sets that are kept invariant by the gradient flow. For these sets, an erroneous initialization may prevent the particles from reaching the support of $\pi$. We provide below a general result holding for any deterministic flow, beyond our specific choice \eqref{eq:ksd_flow} of $v_{\mu_t}$, and thus holding also for SVGD. 

\begin{definition}
Let $\cM\subset \R^d$ be a closed nonempty set. We say that $\cM$ is a flow-invariant support set for the flow $(\mu_t)_{t\ge 0}$ of \eqref{eq:continuity_eq} if for any $\mu_0$ s.t.\ $\supp(\mu_0)\subset \cM$, we have that the flow verifies $\supp(\mu_t)\subset \cM$ for all $t\ge 0$.
\end{definition}

\begin{proposition}\label{prop:stable support}(Informal)
Let $\cM\subset \R^d$ be a smooth nonempty submanifold and $\mu_0\in \cP_c(\X)$ with $\supp(\mu_0)\subset\cM$. Assume that, for a deterministic $(v_{\mu_t})_{t\ge 0}$ satisfying classical Caratheodory-Lipschitz assumptions (\Cref{sec:Cauchy-Lip_flow}), we have $v_{\mu_t}(x)\in T_\cM(x)$ where $T_\cM(x)$ is the tangent space to $\cM$ at $x$. Then $\cM$ is flow-invariant for \eqref{eq:continuity_eq}.
\end{proposition}

The formal statement, \Cref{prop:stable_support_nonsmooth}, stated and proved in \Cref{sec:proof_stable_support}, can be in particular applied to the ubiquitous radial kernels and to planes of symmetry of $\pi$, i.e.\ affine subspaces $\cM\subset \R^d$ such that the density of $\pi$ is symmetric w.r.t.\ $\cM$. \Cref{lem:Flow-invariant symmetry} is illustrated in \Cref{sec:experiments} for a mixture of two Gaussians with the same variance (\Cref{fig:mog_problem}).

\begin{lemma}\label{lem:Flow-invariant symmetry}
Let $\cM$ be a plane of symmetry of $\pi$ and consider a radial kernel $k(x,y)=\phi(\|x-y\|^2/2)$ with $\phi\in C^3(\R)$. Then, for all $(x,y)\in\cM^2$, $\nabla_{2}\kpi(x,y)\in T_\cM(x)$ and $\cM$ is flow-invariant for \eqref{eq:ksd_flow}.
\end{lemma}
\emph{Proof idea:} We show that all the terms in $\nabla_{2}\kpi(x,y)$ belong to $T_\cM(x)$. This implies that the convex combination $\nabla_{W_2}\cF(\mu)(y)=\mathbb{E}_{x \sim \mu}[\nabla_{2}\kpi(x,y)]\in T_\cM(x)$. We then apply \Cref{prop:stable support} to conclude.


\section{Experiments}\label{sec:experiments}


In this section, we discuss the performance of KSD Descent to sample from $\pi$ in practice, on toy examples and real-world problems. The code to reproduce the experiments and a package to use KSD Descent are available at \url{https://github.com/pierreablin/ksddescent}. For all our experiments, we use a Gaussian kernel, as we did not notice any difference in practice w.r.t.\ the IMQ kernel. Its bandwith is selected by cross-validation. Implementation details and additional experiments can be found in \Cref{sec:more_experiments}.

\subsection{Toy examples}\label{sec:experiments_toy}

\begin{figure}
    \centering
    \includegraphics[width=.24\columnwidth]{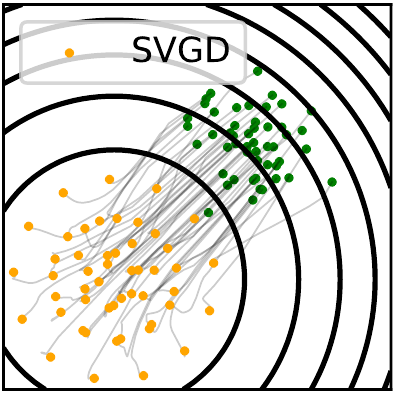}
    \includegraphics[width=.24\columnwidth]{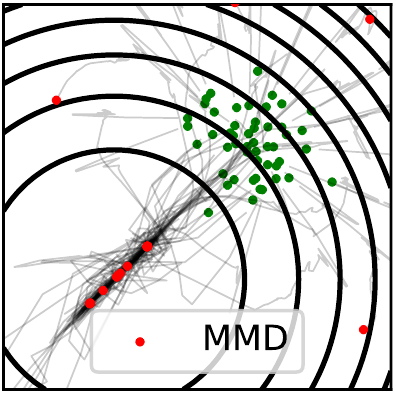}
    \includegraphics[width=.24\columnwidth]{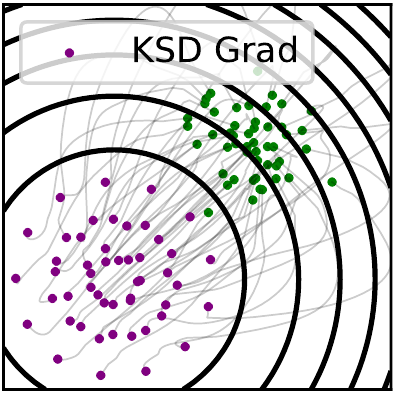}
    \includegraphics[width=.24\columnwidth]{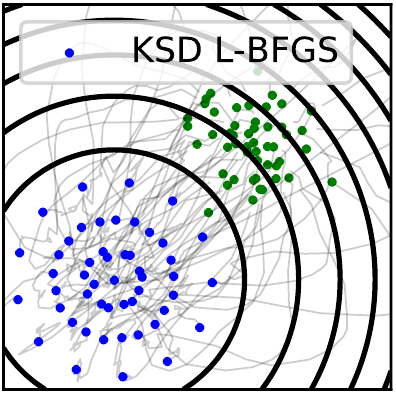}
    \caption{Toy example with 2D standard Gaussian. The green points represent the initial positions of the particles. The light grey curves correspond to their trajectories under the different $v_{\mu_t}$.}
    \label{fig:toy_example}
\end{figure}
In the first example, we choose $\pi$ to be a standard 2D Gaussian, and a Gaussian $k$ with unit bandwidth. We initialize with $50$ particles drawn from a Gaussian with mean $(1,1)$.
\Cref{fig:toy_example} displays the trajectories of several different methods: SVGD, KSD Descent implemented using gradient descent (\Cref{alg:ksd_descent_GD}) and L-BFGS (\Cref{alg:ksd_descent_LBFGS}), and the MMD flow \cite{arbel2019maximum}. To assess the convergence of the algorithms, for SVGD we monitored the norm of the displacement, while for the KSD and MMD gradient flows we used the tolerance parameter of L-BFGS. KSD Descent successfully pushes the particles towards the target distribution, with a final result that is well-distributed around the mode. 
While KSD performs similarly to SVGD, we can notice that the trajectories of the particles behave very differently. Indeed, SVGD trajectories appear to be at first driven by the score term in the update, while the repulsive term acts later to spread the particles around the mode. In contrast, trajectories of the particles updated by KSD Descent are first influenced by the last repulsive term of the update, which seems to determine their global layout, and are then translated and contracted towards the mode under the action of the driving terms. Finally, for the MMD descent, some particles collapse around the mode, while others stay far from the target. This behavior was documented in~\citet{arbel2019maximum}, and can be partly circumvented by injecting some noise in the updates.
 
 \begin{figure}
    \centering
    \includegraphics[width=.99\columnwidth]{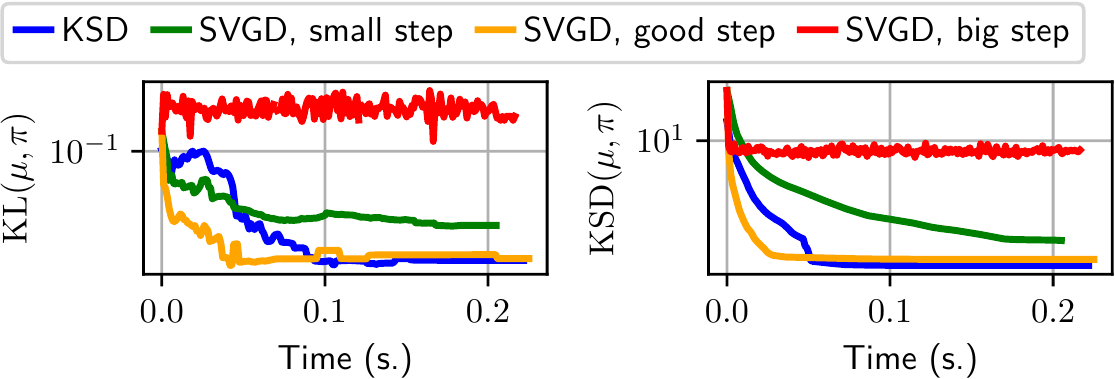}
    \caption{Convergence speed of KSD and SVGD on a Gaussian problem in 1D, with 30 particles.}
    \label{fig:conv_speed}
\end{figure}
 
 We then compare the convergence speed of KSD Descent and SVGD in terms of the KL or KSD objective (see \Cref{fig:conv_speed}). With a fine-tuned step-size, SVGD is the fastest method. However, taking a step-size too large leads to non-convergence, while one too small leads to slow convergence. It should be stressed that it is hard to select a good step-size, or to implement a line-search strategy, since SVGD does not minimize a simple function. In contrast,
 the empirical KSD \eqref{eq:discrete_loss} can be minimized using L-BFGS, which does not have any critical hyper-parameter.
 
\begin{figure}
    \centering
    \includegraphics[width=.48\columnwidth]{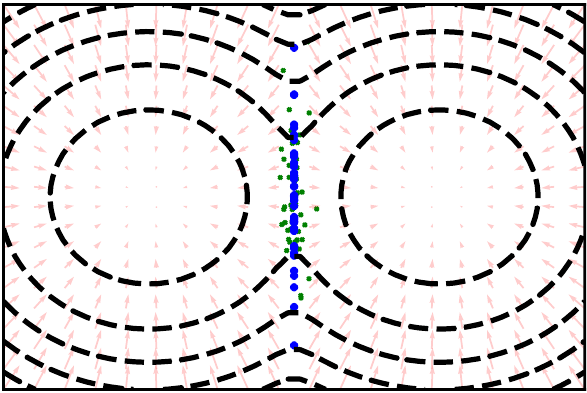}
    \includegraphics[width=.48\columnwidth]{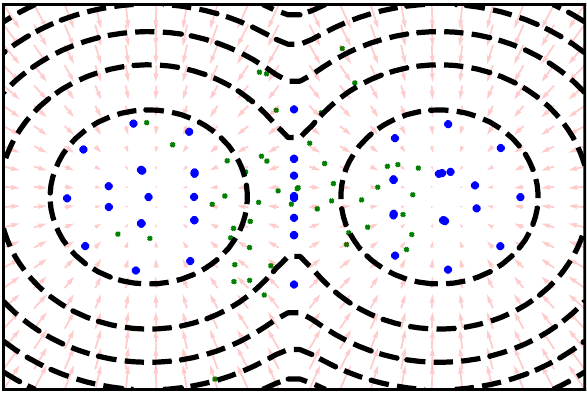}
    \caption{KSD Descent applied on a balanced mixture of Gaussian with small variance (0.1) in 2D. The centroids are at $(-1, 0)$ and $(1, 0)$. The green crosses indicate the initial particle positions, while the blue ones are the final positions. The light red arrows correspond to the score directions. \textbf{Left}: the initial positions are all chosen close to the line $x=0$, which corresponds to an axis of symmetry of $\pi$. \textbf{Right}: even when the initialization is more spread, some particles are still caught in this spurious local minimum.}
    \label{fig:mog_problem}
\end{figure}
 In our second example, we apply KSD Descent for $\pi$ taken as a symmetric mixture of two Gaussians with the same variance. This highlights the results of \Cref{sec:spurious_minima}. If initialized on the axis of symmetry, the particles are indeed stuck on it, as stated in \Cref{lem:Flow-invariant symmetry}. We noticed that, for a large variance of $\pi$ (e.g.\ in $[0.2,1]$), this axis is unstable. However, when the variance is too small (e.g.\ set to 0.1 as in \Cref{fig:mog_problem}), the axis can even become a locally stable set. We also observed that, for a distribution initialized exclusively on one side of the axis, a single component of the mixture can be favored. This is a classical behavior of score-based methods, depending typically on the variance of $\pi$ \cite{wenliang2020blindness}.
 

To fix this issue, we consider an annealing strategy as suggested by \citet{wenliang2020blindness}. It consists in adding an inverse temperature variable $\beta$ to the log density of the model, i.e.\ $\pi^{\beta}(x)\propto \exp(-\beta V(x))$ for $\pi(x)\propto \exp(-V(x))$, with $\beta \in (0, 1]$. This is easily implemented with score-based methods, since it simply corresponds to multiplying $s(x)$ by $\beta$. When $\beta$ is small, annealing smoothes the target distribution and the last term of the Stein kernel, repulsive at short distance, becomes dominant; on the other hand, for $\beta$ close to 1, we recover the true log density. To implement this method, we start with $\beta=0.1$, and run the KSD Descent to obtain particles at `high temperature`. KSD Descent is then re-run starting from these particles, setting now $\beta = 1$. One can see that this strategy successfully solves the issues encountered when the KSD flow was failing to converge to the target $\pi$ (\Cref{fig:toy_example_annealing}).
This correction differs from the noise-injection strategy proposed in \citet{arbel2019maximum} for the MMD flow, which is rather related to randomized smoothing \cite{duchi2012randomized}. Noise-injection would prevent the use of L-BFGS in our case, as it requires exact values of the gradients of previous iterations. Annealing on the other hand is compatible with \Cref{alg:ksd_descent_LBFGS}. 

\begin{figure}
    \centering
    \includegraphics[width=.32\columnwidth]{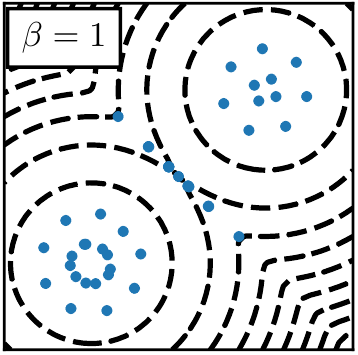}
    \includegraphics[width=.32\columnwidth]{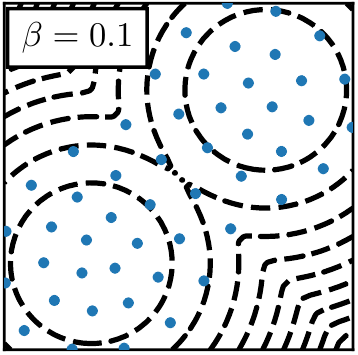}
    \includegraphics[width=.32\columnwidth]{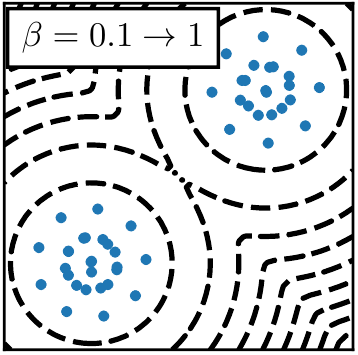}
     \caption{Effect of the annealing strategy on a mixture of Gaussians. \textbf{Left}: without annealing, some particles fall into a spurious minimum. \textbf{Middle}: with a higher temperature ($\beta = 0.1$), the particles are more spread out. \textbf{Right}: starting from the particles in the middle figure and setting $\beta=1$ we converge to a distribution which minimizes the KSD, and has no spurious particles.}
    \label{fig:toy_example_annealing}
\end{figure}
\subsection{Bayesian Independent Component Analysis}
Independent Component Analysis \citep[ICA,][]{comon1994independent} is the generative model $x = W^{-1}s$, where $x$ is an observed sample in $\R^p$, $W\in\R^{p\times p}$ is the unknown square unmixing matrix, and $s\in\R^p$ are the independent sources. We assume that each component has the same density $s_i \sim p_s$. The likelihood of the model is $p(x|W) = \log|W| + \sum_{i=1}^pp_s([Wx]_i)$. For our prior, we assume that $W$ has i.i.d.\ entries, of law $\mathcal{N}(0, 1)$. The posterior is $p(W|x) \propto p(x|W)p(W)$, and the score is given by $s(W) = W^{-\top} - \psi(Wx)x^{\top} - W$, where $\psi = - \frac{p'_s}{p_s}$.
In practice, we choose $p_s$ such that $\psi(\cdot) = \tanh(\cdot)$.
We then use the presented algorithms to draw particles $W\sim p(W|x)$.
We use $N = 10$ particles, and take $1000$ samples $x$ from the ICA model for $p\in\{2,4,8\}$.
Each method outputs $N$ estimated unmixing matrices, $[\tilde{W}_i]_{i=1}^N$.
We compute the Amari distance~\citep{amari1996new} between each $\tilde{W}_i$ and $W$: the Amari distance vanishes if and only if the two matrices are the same up to scale and permutation, which are the natural indeterminacies of ICA.
We repeat the experiment $50$ times, resulting in $500$ values for each algorithm (\Cref{fig:ica}). We also add the results of a random output, where the estimated matrices are obtained with i.i.d.\ $\mathcal{N}(0, 1)$ entries.
\begin{figure}
    \centering
    \includegraphics[width=.32\columnwidth]{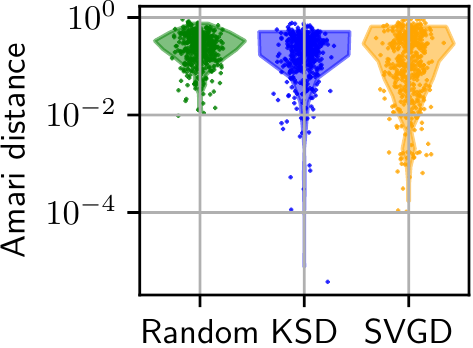}
    \includegraphics[width=.32\columnwidth]{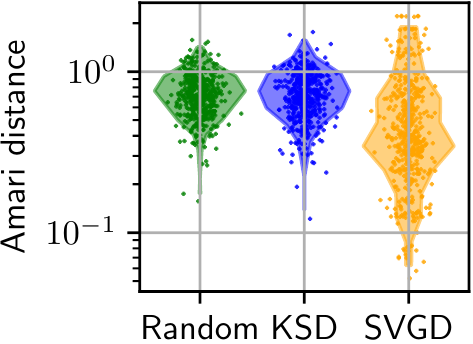}
    \includegraphics[width=.32\columnwidth]{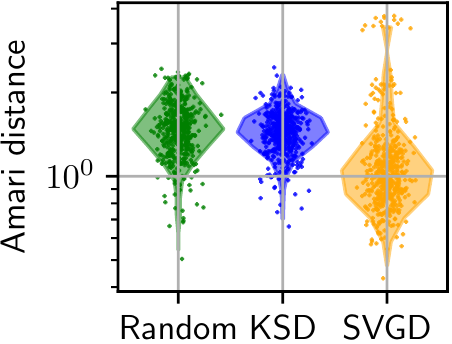}
    \caption{Bayesian ICA results. Left: $p=2$. Middle: $p=4$. Right: $p=8$. Each dot corresponds to the Amari distance between an estimated matrix and the true unmixing matrix.}
    \label{fig:ica}
\end{figure}
We see that for this experiment, KSD performs barely better than random, while SVGD finds matrices with lower Amari distance.
One explanation is that the ICA likelihood is highly non-convex~\citep{cardoso1998blind}. This is easily seen with the invariances of the problem: permuting the rows of $W$ does not change $p(x|W)$. As a consequence, the posterior has many saddle points, in which particles might get trapped. Unfortunately, the annealing strategy proposed above did not improve the achieved performance for this problem.

%

\subsection{Real-world data}

We compare KSD Descent and SVGD in the Bayesian logistic regression setting described in \citet{gershman2012nonparametric, liu2016stein}. Given datapoints $d_1, \dots, d_q \in \R^p$, and labels $y_1, \dots, y_q\in \{\pm 1\}$, the labels $y_i$ are modelled as $p(y_i = 1| d_i, w) = (1 + \exp(-w^{\top}d_i))^{-1}$ for some $w\in \R^p$. The parameters $w$ follow the law $p(w |\alpha) = \mathcal{N}(0, \alpha^{-1}I_p)$, and $\alpha > 0$ is drawn from an exponential law $p(\alpha) = \mathrm{Exp}(0.01)$. The parameter vector is then $x = [w, \log(\alpha)] \in \R^{p+1}$, and we use \Cref{alg:ksd_descent_LBFGS} to obtain samples from $p(x| \left(d_i, y_i)_{i=1}^q\right)$ for $13$ datasets, with $N=10$ particles for each.

The learning rate for SVGD and the bandwidth of the kernel for both methods are chosen through grid-search, and for each problem we select the hyper-parameters yielding the best test accuracy. For all problems, the running times of SVGD with the best step-size and of KSD Descent were similar, while KSD Descent has the advantage of having one less hyper-parameter. We present on \Cref{fig:logreg} the accuracy of each method on each dataset, where KSD Descent was applied without annealing since it did not change the numerical results. Our results show we match the SVGD performance without having to fine-tune the step-size, owing to \Cref{alg:ksd_descent_LBFGS}.
We posit that KSD succeeds on this task because the posterior $p(x| \left(d_i, y_i)_{i=1}^q\right)$ is log-concave, and does not have saddle points.

\begin{figure}[ht]
    \begin{minipage}[c]{0.45\linewidth}
        \includegraphics[width=\textwidth]{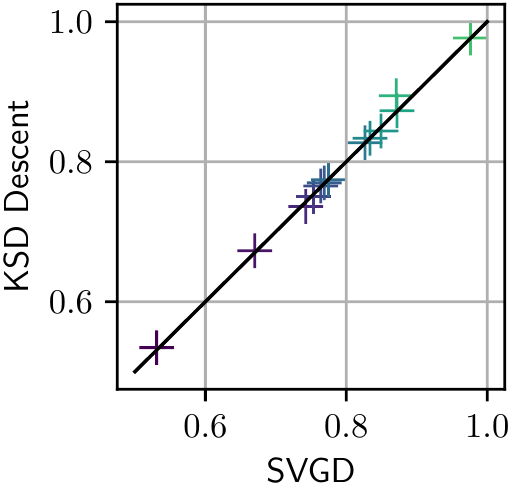}
    \end{minipage}\hfill
    \begin{minipage}[c]{0.45\linewidth}
        \vspace{-1em}
    \caption{Accuracy of the KSD Descent and SVGD on bayesian logistic regression for 13 datasets. Both methods yield similar results. KSD is better by $2\%$ on one dataset.}\label{fig:logreg}
    \end{minipage}
    \vspace{-1em}
\end{figure}

\tb{Discussion.} KSD Descent benefits from a tractable loss and can be straightforwardly implemented with L-BFGS, achieving performance on par with SVGD on convex problems. However its dissipation has non-trivial stationary points, which prevents its use for non-convex problems with saddle-points such as ICA. Convergence of kernel-based sampling schemes is known to be difficult, and we provided some intuitions on the reasons for it. This leaves the door open to a more in-depth analysis of kernel-based gradient flows, especially for unbounded kernels.
\tb{Acknowledgments.} A.K.\ thanks the GENES, and S.M.\ the A.N.R ABSint (Projet-ANR-18-CE40-0034) for the financial support. P.A. thanks A.N.R. ANR19-P3IA-0001.

\bibliography{biblio}
\bibliographystyle{icml2021}

\appendix
\onecolumn
\section{Background}\label{sec:background_app}

\subsection{The continuity equation}\label{sec:continuity_eq}

Let $T>0$. Consider a weakly continuous family of probability measures on $\X$, $\mu : (0, T) \to  \cP_2(\X), t \mapsto \mu_t$. It satisfies a continuity
equation \citep[Section 8.1]{ambrosio2008gradient} if there exists $(v_t)_{t \in (0,T)}$ such that $v_t \in L^2
(\mu_t)$ and :
\begin{equation}
\frac{\partial \mu_t}{\partial t}+\divT(\mu_t v_t ) =0 \qquad \text{ in }\X \times (0,T)
\end{equation}
holds in the sense of distributions, i.e. for any $\phi \in C^{\infty}_c(\X)$, the identity 
\begin{equation}\label{eq:cont_eq_distributions}
\frac{\d}{\d t}  \int_{\X} \phi(x) d\mu_t(x) = \int_{\X} \ps{\nabla \phi(x), v_t(x)} d\mu_t(x)
\end{equation}
holds for any $t \in (0,T)$. By an integration by parts, the r.h.s. of the previous identity can be written:
\begin{equation}
    \int_{\X} \ps{\nabla \phi(x), v_t(x)} d\mu_t(x) = - \int_{\X} \phi(x) \divT(\mu_t(x)v_t(x))\d x  .
\end{equation}
Hence, the identity \eqref{eq:cont_eq_distributions} can be rewritten as
\begin{equation}
    \int_{\X} \phi(x) \frac{\partial \mu_t(x)}{\partial t}\d x + \int_{\X} \phi(x) \divT(\mu_t(x)v_t(x))\d x   = 0. 
\end{equation}

This equation expresses the law of conservation of mass for any volume of a moving fluid. Assume $d=3$ and denote by $(v_t^x, v_t^y, v_t^z)$ the projections of the vector field on the axes $(\overrightarrow{x},\overrightarrow{y},\overrightarrow{z})$. Then, the continuity equation has the form 
\begin{equation}
\frac{\partial \mu_t}{\partial t}+\divT(\mu_t v_t ) = \frac{\partial \mu_t}{\partial t} + \frac{\partial (\mu_t v_t^x) }{\partial x}+\frac{\partial (\mu_t v_t^y) }{\partial y}+\frac{\partial (\mu_t v_t^z) }{\partial z}=0.
\end{equation}

\subsection{Differentiability and convexity on the Wasserstein space }\label{sec:W2_diff}

Let $(\mu,\nu) \in \cP_2(\X)$. The Wasserstein $2$ distance is defined as :
\begin{equation}\label{eq:wass2}
W_2^2 (\mu, \nu) = \inf_{q \in \mathcal{Q}(\mu,\nu)} \int_{\X\times \X} \|x-y\|^2 dq(x,y)
\end{equation}
where $\mathcal{Q}(\mu,\nu)$ is the set of couplings between $\mu$ and $\nu$, \textit{i.e.} the set of nonnegative measures $q$ over $\X \times \X$ such that $P_{1\#} q = \mu$ (resp. $P_{2\#} q = \nu$) where $P_1 : (x,y) \mapsto x$ (resp. $P_2 : (x,y) \mapsto y$).

The Wasserstein space ($\cP_2(\X), W_2)$ is not a  Riemannian manifold but it can be equipped with a Riemannian structure and interpretation \cite{otto2001geometry}. In this geometric interpretation, the tangent space to $\cP_2(\X)$ at $\mu$ is included in $L_2(\mu)$, and is equipped with the scalar product, defined for $f,g \in L_2(\mu)$ by:
\begin{equation}\label{eq:l2_sp}
\ps{f,g}_{L_2(\mu)}=\int_{\X} f(x)g(x)\d\mu(x).
\end{equation}
Let $\cF:\cP_2(\X)\rightarrow \R$ be a functional on the Wasserstein space. We clarify in this section the notions of  differentiability of $\cF$ that we consider in this setting.  The notion of Fréchet subdifferentiability and its properties have been extended to the Wasserstein framework in \citep[Chapter 10]{ambrosio2008gradient}.  We first recall that if it exists, the \textit{first variation of $\cF$ evaluated at $\mu \in \cP_2(\X)$} is the unique function $\frac{\partial{\cF}(\mu)}{\partial{\mu}}
:\X \rightarrow \R$ s.t.\
\begin{equation}\label{eq:first_var}
\lim_{\epsilon \rightarrow 0}\frac{1}{\epsilon}(\cF(\mu+\epsilon  \xi) -\cF(\mu))=\int_{\X}
\frac{\partial{\cF}(\mu)}{\partial{\mu}}(x)
d \xi(x)
\end{equation} 
for all $\xi=\nu-\mu$
, where $\nu \in \cP_2(\X).$ Under mild regularity assumptions, the $W_2$ gradient of $\cF$ corresponds to the gradient of the first variation of $\cF$, as stated below.

\begin{definition}\label{def:w2grad} \citep[Lemma 10.4.1]{ambrosio2008gradient}.
	Let $\mu\in \cP_2(\X)$, absolutely continuous with respect to the Lebesgue measure, with density in $ C^1(\X)$ and such that $\cF(\mu)<\infty$
	. The \textit{subdifferential} of $\cF$ at $\mu$ is the map  $\nabla_{W_2}\cF(\mu)$
	defined by:
	\begin{equation}\label{eq:w2_grad}
	\nabla_{W_2}\cF(\mu)(x)=\nabla \frac{\partial \cF(\mu)}{\partial \mu}(x) \text{ for $\mu$-a.e.  $x \in \X$},
	\end{equation}
	and for every vector field $\xi \in C_c^{\infty}(\X; \X)$,
	\begin{equation}
	    \int_{\X} \ps{\nabla_{W_2}\cF(\mu)(x), \xi(x)} \d\mu(x) = -\int_{\X} \frac{\partial \cF(\mu)}{\partial \mu}(x) \divT(\mu(x) \xi(x))\d x  .
	\end{equation}
	Moreover, $\nabla_{W_2}\cF(\mu)$ belongs to the tangent space of $\cP_2(\X)$ at $\mu$, which is included in $L^2(\mu)$.
\end{definition}

\subsection{Cauchy-Lipschitz assumptions for the existence and uniqueness of the Wasserstein gradient flow}\label{sec:Cauchy-Lip_flow}

Let $T>0$, and denote by $\cL^{1}$ the standard one-dimensional Lebesgue measure on $[0, T]$ and $L^1$ the space of measurable and integrable functions w.r.t.\ Lebesgue measure. The Cauchy-Lipschitz assumptions below for existence and uniqueness of the flow on $[0, T]$ are adapted to our flows from their general differential inclusion version by \citet{Bonnet2021DiffInclWass}. They hold for an initial distribution $\mu_0\in\cP_c(\R^d)$, the space of probability measures with compact support. If $v$ only depends on time and not on $\mu_t$, they then write as
\newcounter{contlist_CL}
\begin{assumplist2}
	\setlength\itemsep{0.2em}
	\item \label{ass:caratheodory_field} (Carathéodory vector fields) $v:[0, T] \times$ $\R^{d} \rightarrow \R^{d}$ is such that $t \mapsto v(t, x)$ is $\cL^{1}$-measurable for all $x \in \R^{d}$ and $x \mapsto v(t, x)$ is continuous for $\cL^{1}$-almost every $t \in[0, T]$.
	\item \label{ass:sublinear_growth} (Sublinear growth) There exists a map $m(\cdot) \in L^{1}\left([0, T]; \R_{+}\right)$ such that $|v(t, x)| \leq m(t)(1+|x|)$ for $\cL^{1}$-almost every $t \in[0, T]$ and all $x \in \R^{d}$.
	\item \label{ass:lip_field} (Lipschitz vector field) For any compact set $\cK \subset \R^{d},$ there exists a map $l_{\cK}(\cdot) \in L^{1}\left([0, T], \R_{+}\right)$ such that $\Lip(v(t, \cdot) ; \cK) \leq l_{\cK}(t)$ for $\cL^{1}$ -almost every $t \in[0, T]$.
\end{assumplist2}
\setcounter{contlist_CL}{\value{enumi}}
For $R>0$, we denote by $\cK:=B(0, R)$ the closed ball of radius $R$ in $\X$; and for any function $f$, we denote by $\|f(\cdot)\|_{\infty,\cK}:=\sup_{x\in \cK}|f(x)|$ its supremum over $\cK$ and by $\Lip(f(\cdot) ; \cK)$ the Lipschitz constant of the restriction of $f$ on $\cK$. 
If $v$ only depends on $\mu_t$, then the following assumptions should hold for every $R>0$ :
\begin{assumplist3}
	\setlength\itemsep{0.2em}
	\item \label{ass:caratheodory_field_contr} For any $\mu \in \cP_c\left(\R^{d}\right)$,  $v_{|\cK}(\mu)(\cdot)\in C^{0}\left(\cK; \R^{d}\right)$.
	\item \label{ass:sublinear_growth_contr} There exists $m>0$ such that for any $\mu \in \cP_c\left(\R^{d}\right),$ for all $y \in \R^{d},$ we have $\|v(\mu)(y)\| \leq m\left(1+\|y\|+\int \|x\|\,\d\mu(x)\right)$
	\item \label{ass:lip_field_contr} There exists $l_\cK>0$ such that for any $\mu \in \cP_c(\cK)$, we have $\Lip(v(\mu)(\cdot) ; \cK) \leq l_{\cK}$,
	\item \label{ass:lip_field_wrtMeasure_contr} There exists $L_\cK>0$ such that for any $\mu, \nu \in \cP_c(\cK),$ we have $\|v_{|\cK}(\mu)(\cdot)-v_{|\cK}(\nu)(\cdot)\|_{\infty,\cK}\le L_\cK W_{2}(\mu, v)$.
\end{assumplist3}
\setcounter{contlist_CL}{\value{enumi}}
\tb{Relation with \Cref{ass:lipschitz}.} Consider $\mu \in \cP_c\left(\R^{d}\right)$ and take $R>0$ such that $\supp(\mu)\subset \cK:=B(0, R)$.
Our \Cref{ass:lipschitz} implies that $y\mapsto v_\mu(y)=\int \nabla_2 k_{\pi}(x,y) \d\mu(x)$ is Lipschitz with constant $l_{\cK}=\sup_{y\in\cK}L(y)$, so \Cref{ass:caratheodory_field_contr,ass:lip_field_contr} hold. \Cref{ass:sublinear_growth_contr} corresponds to \Cref{ass:growth}. Finally, for $\nu \in \cP_c(\cK)$, since $\nabla_2 k_{\pi}(\cdot,y)$ is $l_{\cK}$-Lipschitz for $y\in\cK$, we have that $x\mapsto\|\nabla_2 k_{\pi}(x,y)\|$ is also $l_{\cK}$-Lipschitz, hence
\begin{align*}
\left\|v_\mu(y)-v_\nu(y)\right\|\le \sup \left\{\int_{M} f(x) \mathrm{d}(\mu-\nu)(x) \mid \text {Lipschitz}\, f: \cK \rightarrow \mathbb{R}, \Lip(f) \leq l_\cK\right\}= l_\cK W_{1}(\mu, \nu) \le l_\cK W_{2}(\mu, \nu),
\end{align*}
where the last inequality between 1-Wasserstein distance $W_1$ and $W_2$ is a consequence of Jensen's inequality, so \Cref{ass:lip_field_wrtMeasure_contr} is satisfied.

For other kernel-based updates as presented in \Cref{sec:updates_related_work}, the kernel is at most multiplied once by the score $s$, so for uniformly Lipschitz $s$ and kernels with bounded derivatives, the Lipschitz constant is uniform for $v(\mu)(\cdot)$. This was for instance assumed in \citet{arbel2019maximum}. However for updates involving the Stein kernel such as in KSD Descent, the analysis is more intricate. As a matter of fact, when $\pi$ is a standard Gaussian distribution, we have $s(x)=x$. Hence the Stein kernel verifies, for a smooth translation-invariant kernel with $k(x,x)=1$, that $\kpi(x,x)=C+\|x\|^2$ for $x\in \X$, with $C$ a constant determined by the last term in $\kpi$. So, for a Gaussian $\pi$, the Lipschitz constant of $v(\delta_y)(\cdot)$ is larger than $2\|y\|$, explaining the reason why we only required for a $y$-dependent Lipschitz constant in \Cref{ass:lipschitz} and did not assume it to be uniform. This is related to the fact that $\kpi(x,x)$ is in general unbounded for the Stein kernel, which precludes most of the classical assumptions made in the kernel literature.

\tb{Comment on \Cref{ass:sublinear_growth_contr}.}  The main role of \Cref{ass:sublinear_growth_contr} is to impede the well-known phenomenon of finite-time explosion of a trajectory as for the system $z'(t)=z(t)^2$ in $\R$. It could in principle be replaced by any assumption preventing particles from escaping to infinity in finite time. One such lighter assumption would write as follows:
\begin{center}
	`` For any $r>0$, $T>0$ and $\mu_0 \in \cP_c\left(\R^{d}\right)$ with $\supp(\mu_0)\subset rB(0,1)$ there exists $R(r,T)<\infty$ such that $\supp(\mu_t)\subset R(r,T)B(0,1)$ for $(\mu_t)_{t\in[0,T]}$ solving \eqref{eq:ksd_flow}.``
\end{center}

\subsection{Descent updates of the algorithms presented in \Cref{sec:related_work}}\label{sec:updates_related_work}
From a computational viewpoint, both the SVGD, LAWGD and MMD flow algorithms, presented in \Cref{sec:related_work}, only propose gradient descent schemes with the following respective updates $x_n^i \leftarrow x_n^i - \gamma \mathcal{D}^i$ with
\begin{align*}
&\cD_{SVGD}^i = \frac1N\sum_{j=1}^N\left[k(x_n^j, x_n^i)s(x_n^j) + \nabla_{1} k(x_n^j, x_n^i)\right],\\
&\cD_{LAWGD}^i = \frac{1}{N}\sum_{j=1}^N \nabla_{2}k_{\cL_\pi}( x_n^{j},x_n^{i}),\\
&\cD_{MMD-GD}^i =\frac{1}{N}\sum_{j=1}^{N}\left[\nabla_2 k(x_{n}^{j}, x_{n}^{i})-\nabla_2 k(y^{j}, x_{n}^{i})\right].
\end{align*}
where the MMD-GD update requires extra samples $(y^{j})_{j=1}^N\sim \pi$, since it is sample-based rather than score-based. Notice that all these updates have the same iteration complexity of order $\mathcal{O}(N^2)$. Intriguingly LAWGD has the same update rule as ours \eqref{eq:ksd_update} but for their kernel $k_{\cL_\pi}$ which, as discussed in \Cref{sec:related_work}, does not incorporate an off-the-shelf kernel $k$, unlike the Stein kernel $\kpi$ \eqref{eq:stein_kernel}. 


\subsection{Background on diffusion operator $\cL_{\pi}$}\label{sec:background_diffusion}

In this section, for the convenience of the reader who is not familiar with the spectral theory of diffusion operators, we formulate \Cref{lem:diffusion_background} which provides a formal construction of the diffusion operator $\cL_{\pi} = - \ps{\nabla \log\pi, \nabla} - \Delta$ on the space $L_2(\pi)$ and gathers a number of technical facts about it. Those facts form a background for the proofs of results related to lack of exponential convergence near equilibrium of KSD flow, which are presented in \Cref{appendix:proofs}.  We provide the proof of \cref{lem:diffusion_background} in \Cref{sec:proof_diffusion_background}.

\begin{lemma} \label{lem:diffusion_background}
Let $\pi \propto e^{-V}$ be a probability measure on $\R^d$ and assume that $V : \R^d \rightarrow \R$ is in $C^1(\X)$. Let $\hat{\cL}_{\pi} = \ps{\nabla V, \nabla} - \Delta$  on $C_c^{\infty}(\R^d)$. This operator can be extended to a  positive self-adjoint operator on $L^2(\pi)$ with dense domain $\mathcal{D}(\cL_{\pi})\subset L^2(\pi)$ which we denote by $\cL_{\pi}$. Moreover,  $C_c^{\infty}(\X)$ is dense in $\mathcal{D}(\cL_{\pi})$, for the norm:
\begin{equation} \label{eq:bilinear_form_norm}
    \Vert \phi \Vert_{\cL_{\pi}} = \left( \langle \phi , \cL_{\pi} \phi \rangle + \Vert \phi \Vert_{L_2(\pi)}^2 \right)^{1/2}.
\end{equation}
From that it follows, that $\mathcal{D}(\cL_{\pi})$ is the subset of the weighted Sobolev space $W_0^{1,2}(\pi)$\footnote{
Note that the meaning of $0$ in the notation $W_0^{1,2}(\pi)$ differs from the meaning of $0$ in $L^2_{0}(\pi)=\{\phi \in L^2(\pi), \int \phi d\pi = 0\}$. } (that is, the closure of $C_c^{\infty}(\R^d)$ in $W^{1,2}(\pi)$ ) and for all $f \in \mathcal{D}(\cL_{\pi})$ we have 
\begin{equation} \label{eq:dirichlet_form_pi}
    \Vert \nabla f \Vert_{L^2(\pi)}^2 = \langle f, \cL_{\pi} f \rangle_{L^2(\pi)},
\end{equation}
where $\nabla f$ is the weak derivative of $f$. This implies that the kernel of $\cL_{\pi}$ consists of $\pi$-almost everywhere constant functions.

Furthermore, for any $f \in \mathcal{D}(\cL_{\pi})$ we can find a sequence $\phi_n \in C_c^{\infty}(\R^d)$, such that $\lim_{n \rightarrow \infty} \Vert \phi_n - f \Vert_{L_2(\pi)} = 0$ and $\lim_{n \rightarrow\infty} \Vert \cL_{\pi} \phi_n - \cL_{\pi} f \Vert_{L_2(\pi)} = 0$.

\end{lemma}

\subsection{A Descent lemma for KSD Descent}\label{sec:descent_lemma}

A descent lemma for \eqref{eq:ksd_descent} is a proposition stating that $\cF$ decreases at each iteration of the time-discretized  flow. It should hold for both continuous or discrete initializations. \citet{arbel2019maximum} proved a similar result for $\cF=\frac{1}{2}\MMD^2$ which we recall below for completeness.
\begin{proposition}\citep[Proposition 4]{arbel2019maximum}\label{prop:descent_lemma_mmd} Let $\mu_n$ be defined by \eqref{eq:ksd_descent} for $\cF=\frac{1}{2}\MMD^2$ and assume that $k\in C^{1,1}(\X\times\X)$ with $l$-Lipschitz gradient: $\left\|\nabla k\left(x, x'\right)-\nabla k\left(y, y'\right)\right\| \leq$ $l\left(\|x-y\|+\left\|x'-y'\right\|\right)$ for all $x, x', y, y' \in \X$. Then, for $\gamma \le \frac{2}{l}$, the sequence $(\cF(\mu_n))_{n\ge 0}$ is decreasing and for any $n\ge0$:
	\begin{equation}
	\cF(\mu_{n+1}) - \cF(\mu_n) \le -\gamma\left(1 - \frac{3\gamma l}{2}\right)  \|\nabla_{W_2}\cF(\mu_n)\|^2_{L^2(\mu_n)}.
	\end{equation}
\end{proposition}
Although KSD Descent is a special case of MMD Descent, assuming a uniform Lipschitz constant for $\kpi$ is excessive, even for targets as simple as a single Gaussian distribution, as discussed in \Cref{sec:Cauchy-Lip_flow}.  In contrast, we consider the weaker \Cref{ass:lipschitz} on the Stein kernel: there exists a map $L(\cdot)\in C^0(\X;\R_+)$ such that, for any $y \in \X$, the maps $x\mapsto \nabla_1 \kpi(x,y)$ and $x\mapsto \nabla_1 \kpi(y,x)$ are $L(y)$-Lipschitz. 

We shall assume below the convexity of $L(\cdot)$ and the boundedness of its $L^2$-norm along the trajectory, i.e.\ $(\|L\|_{L^2(\mu_n)})_{n\ge0}$ is bounded. An example of more explicit, though tighter, assumptions are that $L$ satisfies a subpolynomial growth at infinity, i.e.\ there exists $R>0$, $c>0$ and $m\in\N$ such that $L(x)\le \tilde{L}(x)=c \|x\|^m$ for $x\in\R^d$ with $\|x\|\ge R$, combined with boundedness of the corresponding moment along the flow, i.e.\ there exists $M_m$ such that $\mathbb{E}_{x\sim \mu_n}\|x\|^{2m}\le M_m$. In this case, instead of the continuous $L$, one can consider the convex and continuous function defined by $\tilde{L}(x)=c \|x\|^m+\sup_{y\in B(0,R)} L(y)$.

\begin{proposition}\label{prop:descent_lemma_appendix} Suppose \Cref{ass:lipschitz} holds and that $\mu_0\in\cP_c(\R^d)$. Assume additionally that $L(\cdot)$ is convex, belongs to $L^2(\mu_n)$, and that $\|L(\cdot)\|_{L^2(\mu_n)}\le M$ for any $n\ge 0$, where $\mu_n$ is defined by \eqref{eq:ksd_descent}. Then, for any $\gamma \le \frac{1}{M}$:
	\begin{equation}
	\cF(\mu_{n+1}) - \cF(\mu_n) \le -\gamma\left(1 - \gamma M\right)  \|\nabla_{W_2}\cF(\mu_n)\|^2_{L^2(\mu_n)}\le 0.
	\end{equation}
\end{proposition}
See the proof in \Cref{sec:descent_lemma_proof}.
\section{Proofs} \label{appendix:proofs}
We shall often use that, for smooth and symmetric $k$, $\nabla_2 k(x,y)=\nabla_1 k(y,x)$, for any $x,y\in\X$. Moreover if $k$ is translation-invariant, $\nabla_1 k(x,y)=-\nabla_2k(x,y)$.
\begin{remark}\label{rk:grad_kpi_formula}
For any $k\in C^2(\R^d\times \R^d,\R)$ and $\pi$ such that $s \in C^1(\R^d)$, $\kpi$ and its gradient are written
\begin{align}
\kpi(x,y)&=s(x)^\top s(y) k(x,y)+s(x)^\top \nabla_2 k(x,y) + s(y)^\top \nabla_1 k(x,y)+  \div_{1}\nabla_{2}k(x,y). \label{eq:kpi_formula}\\
\nabla_2 \kpi(x,y)&=s(x)^\top s(y) \nabla_2 k(x,y)+Js(y)^\top s(x) k(x,y) +H_{2}k(x,y)s(x)\nonumber\\
&\qquad+Js(y)^\top \nabla_1 k(x,y)+\nabla_2 ( \div_{1}\nabla_{2}k(x,y)). \label{eq:nabla2_kpi}
\end{align}
If $k$ is translation-invariant, since $\nabla_1k(x,y)=-\nabla_2k(x,y)$, we have:
\begin{align}
\kpi(x,y)&=s(x)^\top s(y) k(x,y)+(s(x)-s(y))^\top \nabla_2 k(x,y) +  \div_{1}\nabla_{2}k(x,y), \label{eq:kpi_formula_2} \\
\nabla_2 \kpi(x,y)&=s(x)^\top s(y) \nabla_2 k(x,y)+Js(y)^\top s(x) k(x,y) +H_{1}k(x,y)s(x)\nonumber\\
&\qquad-Js(y)^\top \nabla_2 k(x,y)-\nabla_2 ( \div_{2}\nabla_{2}k(x,y)). \label{eq:nabla2_kpi_2}
\end{align}
Notice that, in this case, odd derivatives of $k$ vanish for $x=y$.
\end{remark}
\subsection{Proof of \Cref{lem:lipschitz_ass}}\label{sec:proof_lem_lipschitz_ass}
To prove that \Cref{ass:lipschitz} holds, we want to show that $x\mapsto\nabla_1\kpi(x,y)$ and $x\mapsto\nabla_2\kpi(x,y)$ are $L(y)$-Lipschitz for some $L(\cdot)\in C^0(\X,\R_+)$ that is $\mu$-integrable for any $\mu\in \cP_2(\X)$, so $L$ should have a quadratic growth at most. We will leverage the regularity of $k$ and $s$ to write this problem as that of upper-bounding over $\X$ the gradients of these functions (which are related to the Hessian of $\kpi$). To prove that \Cref{ass:integrable} holds, we want on the other hand to identify an integrable function that upper bounds $\| H_1\kpi(x,y)\|_{op}$.

We denote by $Js(x)$ the Jacobian of $s$ at $x$, which is also the Hessian of $\log(\pi)$, and by $Hs(x)$ the Hessian of $s$ at $x$, which is also the tensor of third derivatives of $\log(\pi)$, i.e.\ $Hs(x)_{ijk} = \frac{\partial^3 \log(\pi)}{\partial x_i \partial x_j \partial x_k}(x)$.

We compute the gradient of the first term in \eqref{eq:kpi_formula}:
\begin{equation}
\nabla_x\left(s(x)^{\top}s(y) k(x, y)\right)= Js(x)^\top s(y)k(x, y) + s(x)^{\top}s(y)\nabla_1 k(x, y). 
\end{equation}
Differentiating a second time, we get that
\begin{align}
H_1\left(s(x)^{\top}s(y) k(x, y)\right)&= Hs(x)s(y)k(x, y) + \nabla_1 k(x, y) s(y)^\top Js(x)+ \left(Js(x)^\top s(y)\right)\nabla_1 k(x, y)^{\top}+  s(x)^{\top}s(y) H_1 k(x, y),\nonumber
\end{align}
where $Hs(x)s(y)$ is a $d \times d$ matrix of entries
$[Hs(x)s(y)]_{jk} = \sum_{i=1}^nHs(x)_{ijk}s(y)_i$.

Then, the gradient of the second term in \eqref{eq:kpi_formula} is
\begin{align}
\nabla_x\left(s(x)^{\top}\nabla_2 k(x, y)\right)&= Js(x)^\top\nabla_2 k(x, y) + \nabla_{1,2}k(x, y)s(x),
\end{align}
where $\nabla_{1,2}k(x, y)=[\partial_{x_i}\partial_{y_j}k(x, y)]_{i,j}$, and its Hessian is given by:
\begin{multline}
H_1\left(s(x)^{\top}\nabla_2 k(x, y)\right)= Hs(x)\nabla_2 k(x, y) + Js(x)^\top\nabla_{1,2}k(x, y) + \nabla_{1,2} k(x, y)Js(x)
 + \nabla_{1,1,2}k(x, y) s(x),
 \end{multline}
 where $\nabla_{1,1,2}k(x, y)=[\partial_{x_i,x_j}\partial_{y_l}k(x, y)]_{i,j,l}$ is a tensor of third derivatives of $k$. 
 
The Hessian of the last two terms in \eqref{eq:kpi_formula} is straightforward to compute. Hence, collecting all the terms together, we derive that
\begin{align}
 H_1 \kpi(x,y) =& Hs(x)s(y)k(x, y) + \nabla_1 k(x, y) s(y)^\top Js(x)+ \left(Js(x)^\top s(y)\right)\nabla_1 k(x, y)^{\top}+  s(x)^{\top}s(y) H_1 k(x, y) \label{eq:comp_H1kpi}\\ 
 & Hs(x)\nabla_2 k(x, y) + Js(x)^\top\nabla_{1,2}k(x, y) + \nabla_{1,2} k(x, y)Js(x)
 + \nabla_{1,1,2}k(x, y) s(x) \nonumber\\
 &+\nabla_{1,1,1}k(x, y) s(y)+ H_1\Tr(\div_{2}\nabla_{1}k(x, y)). \nonumber
\end{align}
In this expression, the only problematic terms to upper bound $H_1 \kpi$ uniformly w.r.t.\ $x$ are the ones where $s(x)$ appears (since $Js(x)$ and $Hs(x)$ are bounded by assumption). However, since $s$ is $C_1$-Lipschitz, we have that $\|s(x)\|\le \|s(y)\| + C_1 \|x-y\|$.  We then use that the kernel, its derivatives, and the derivatives up to order 3 multiplied by $\|x-y\|$ are bounded for all $x$ and $y$ by a given $B\ge 0$. Hence, for instance,
\begin{align*}
\|s(x)^{\top}s(y) H_1 k(x, y) \|\le \|s(y)\|\cdot(\|s(y)\|+C_1 \|x-y\|)\cdot\|H_1 k(x, y) \|\le B \|s(y)\|^2 + C_1 B \|s(y)\|.
\end{align*}
As $\|  Js(x) \|_{op}\le C_1$ and $\|  Hs(x) \|_{op}\le C_2$ for all $x\in\X$, we have that $\|s(x) \|\le C_1\|x\|+\|s(0) \|$ and a similar computation gives that
\begin{align}
  \|  H_1 \kpi(x,y) \|_{op} &\le C_2 B \|s(y)\| + 2 C_1 B \|s(y)\|+ B \|s(y)\|^2 + C_1 B \|s(y)\|\\
  &\hspace{5mm}+ C_2 B + 2C_1 B+B\|s(y)\|+C_1 B+B\|s(y)\| +B \nonumber\\
  &\le \tilde{C}_2 \|y\|^2 + \tilde{C}_0, \label{eq:def_Ly_boundedHess}
\end{align}
where $\tilde{C}_2\ge 0$ and $\tilde{C}_0\ge 0$ only depend on $B$, $C_2$, $C_1$ and $\|s(0) \|$. Consequently, the function $(x,y)\mapsto  \|  H_1 \kpi(x,y) \|_{op}$ is integrable for any $\mu,\nu\in \cP_2(\X)$, so \Cref{ass:square_der_int} holds.\footnote{We could do without the assumption of Lipschitzianity of $s$ to prove that \Cref{ass:square_der_int} holds. We could also merely require that the derivatives of $k$ are bounded, without considering the multiplication by $\|x-y\|$. As a matter of fact, assuming that $\|  Hs(x) \|_{op}\le C_2$ for all $x\in\X$ implies that $\|  Js(x) \|_{op}\le C_2\|x\|+C_1$ and $\|s(x) \|\le C_2\|x\|^2+C_0$ for some $C_0\ge 0$ and $C_1\ge 0$. Hence \eqref{eq:comp_H1kpi} yields an upper bound that is quadratic in $\|x\|$ and $\|y\|$, so $\mu\otimes\nu$-integrable as in \Cref{ass:square_der_int}.}

Similarly, one can show based on \eqref{eq:nabla2_kpi} that
\begin{align*}
 \nabla_{1,2} \kpi(x,y) =& (Js(x)^\top s(y))  \nabla_2 k(x, y)^\top + s(x)^\top s(y) \nabla_{1,2} k(x,y) + Js(y)^\top Js(x) k(x,y)+ \nabla_1 k(x, y) s(x)^{\top} Js(y)\\
 &+ H_2 k(x, y) Js(x) + \nabla_{1,2,2}k(x, y) s(x)+ Js(y)^\top H_1 k(x, y)+ \nabla_{1,2}\Tr(\div_{2}\nabla_{1}k(x, y)).
\end{align*}
Using the same inequalities that led to \eqref{eq:def_Ly_boundedHess}, we can find two constants $\hat{C}_2\ge 0$ and $\hat{C}_0\ge 0$ such that
\begin{align*}
 \|\nabla_{1,2} \kpi(x,y)\|_{op} \le& \hat{C}_2 \|y\|^2 + \hat{C}_0. \label{eq:def_Ly_boundedHess}
\end{align*}
Setting $L(y)=\max(\tilde{C}_2, \hat{C}_2) \|y\|^2 + \max(\tilde{C}_0, \hat{C}_0)$, the functions $x\mapsto\nabla_1\kpi(x,y)$ and $x\mapsto\nabla_2\kpi(x,y)$ are $L(y)$-Lipschitz since $L(y)$ is an upper bound of the norm of their Jacobians. Furthermore, $L(\cdot)\in C^0(\X,\R_+)$, and is $\mu$-integrable for any $\mu \in \cP_2(\X)$, so \Cref{ass:lipschitz} holds. 

To prove that \Cref{ass:integrable} holds, notice that, based on \eqref{eq:kpi_formula},
\begin{align*}
 |\kpi(x,x)|_{op} =& \left| \|s(x)\|^2 k(x,x)+s(x)^\top \nabla_2 k(x,x) + s(x)^\top \nabla_1 k(x,x)+  \div_{1}\nabla_{2}k(x,x) \right|\\
 \le & B\left( \|s(x)\|^2 + 2 \|s(x)\| + 1 \right).
\end{align*}
Since,  for all $x\in\X$, $\|s(x) \|\le C_1\|x\|+\|s(0) \|$, the function $x\mapsto \kpi(x,x)$ is bounded by a quadratic function and is thus $\mu$-integrable for any $\mu\in \cP_2(\X)$. This shows that \Cref{ass:integrable} holds.

We now prove that \Cref{ass:growth} is also satisfied if there exists $M>0$ and $M_0$ such that,  for all $x\in\X$, $\|s(x) \|\le M\sqrt{\|x\|}+M_0$. Then, based on \eqref{eq:nabla2_kpi} and as $\sqrt{\|x\|}\le \|x\| +1$,
\begin{align*}
\|\nabla_2 \kpi(x,y)\|&\le \|s(x)\|\cdot \|s(y)\| \cdot\|\nabla_2 k(x,y)\|+ B C_1 \| s(x) \| + B \|s(x) \|+B C_1+B\\
&\le (\|s(y)\|+C_1 \|x-y\|)\cdot \|s(y)\| \cdot\|\nabla_2 k(x,y)\| + (B C_1+B)\cdot(M\sqrt{\|x\|}+M_0+1)\\
&\le B\|s(y)\|^2 + B C_1 \|s(y)\| + (B C_1+B)\cdot(M\|x\| +M+M_0+1)\\
&\le B (M\sqrt{\|y\|}+M_0)^2 + B C_1 (M\|y\| +M+M_0)+ (B C_1+B)\cdot(M\|x\| +M+M_0+1)\\
&\le m(1+\|y\|+\|x\|),
\end{align*}
for some $m>0$ depending on $B$, $M$, $C_1$ and $M_0$. Integrating over $\mu$ for any $\mu \in \cP_2(\X)$ gives that
$\|\int \nabla_2 k_{\pi}(x,y) \d\mu(x)\| \leq m\left(1+\|y\|+\int \|x\|\,d\mu(x)\right)$, so \Cref{ass:growth} holds, which concludes the proof.

\subsection{Proof of \Cref{prop:dissipation}}

\begin{lemma}
Suppose \Cref{ass:lipschitz} holds. The first variation of $\cF$ evaluated at $\mu \in \cP_2(\X)$ is then the function defined for any $y \in \X$ by:
\begin{equation*}
\frac{ \partial{\cF}(\mu)}{\partial{\mu}}(y)=\int_{\X} k_{\pi}(x,y)\d\mu(x).
\end{equation*}
\end{lemma}

\begin{proof}
Let  $\mu,\nu \in \cP_2(\X)$, and $\xi=\mu-\nu$
. Using the symmetry of $\kpi$, we get
\begin{align*}
\frac{1}{\epsilon}\left[\cF(\mu+\epsilon  \xi) -\cF(\mu)\right]&=\frac{1}{2\epsilon}\left[ \iint_{X} \kpi(x,y)d(\mu+\epsilon  \xi)(x) d(\mu+\epsilon  \xi)(y) -  \iint_{X} \kpi(x,y)d\mu(x) d\mu(y) \right]\\
&=  \iint_{\X} \kpi(x,y)d\mu(x) d\xi(y) + \frac{\epsilon}{2} \iint_{X} \kpi(x,y) d\xi(x)d\xi(y)
\end{align*}
Hence, 
\begin{equation*}
    \lim_{\epsilon \rightarrow 0}\frac{1}{\epsilon}(\cF(\mu+\epsilon  \xi) -\cF(\mu)) =  \iint_{\X} \kpi(x,y)d\mu(x) d\xi(y) . \qedhere
\end{equation*}
\end{proof}

Assume $\mu$ satisfies the assumptions of \Cref{def:w2grad}. To obtain the expression for the $W_2$ gradient of $\cF$, we first need to 
exchange the integration and the gradient with respect to $y$. Since $y\mapsto\nabla_2 k_{\pi}(x,y)$ is $L(x)$-Lipschitz, we have for all $v \in \R^d, 0 < \epsilon \leq 1$:
\begin{equation*}
\left\vert \frac{k_{\pi}(x, y+ \epsilon v) - k_{\pi}(x, y)}{\epsilon} - \langle \nabla_2 k_{\pi}(x,y), v \rangle \right\vert \leq \frac{L(x) \epsilon \Vert v \Vert_2^2}{2} \leq \frac{L(x) \Vert v \Vert_2^2}{2}
\end{equation*}
where the right hand side is integrable. Therefore by the Lebesgue dominated convergence theorem, we have the following interchange of the gradient and the integral when computing the $W_2$-gradient:
\begin{equation}\label{eq:explict_W2KSDD}
\nabla_{W_2}\cF(\mu)(y)=\nabla \frac{ \partial{\cF}(\mu)}{\partial{\mu}}(y) = \int \nabla_2 k_{\pi}(x,y) \d\mu(x).
\end{equation}

Then, using the continuity equation \eqref{eq:ksd_flow} and an integration by parts,\footnote{By the regularity of $\mu$, from the assumptions of \Cref{def:w2grad}, there are no boundary terms in the integration.} the dissipation of $\cF$ along its $W_2$ gradient flow is obtained as follows:
\begin{equation*}
\frac{\d\cF(\mu_t)}{\d t}=\int \frac{\partial \cF(\mu_t)}{\partial \mu_t} \frac{\partial \mu_t}{\partial t}=\int \frac{\partial \cF(\mu_t)}{\partial \mu} \divT\left( \mu_t \nabla \frac{\partial \cF(\mu_t)}{\partial \mu}\right)=-\int \left\|\nabla \frac{\partial \cF(\mu_t)}{\partial \mu}\right\|^2 \d \mu_t.
\end{equation*}
Plugging \eqref{eq:explict_W2KSDD} in the r.h.s.\ of this equality leads to the final formula.

\subsection{Proof of \Cref{prop:hessian_ksd_mu}}\label{sec:proof_hessian_ksd_mu}

We compute the second time derivative $\ddot{\cF}(\rho_{t})$
where $\rho_t$ is a path from  $\mu$ to $(I+\nabla\psi)_{\#}\mu$ given by: $\rho_t=  (I+t\nabla\psi)_{\#}\mu$, for all $t\in [0,1]$. 
By \Cref{lem:derivative_ksd}, $\ddot{\cF}(\rho_t)$ is well-defined and $\ddot{\cF}(\mu) =\ddot{\cF}(\rho_t)|_{t=0}$ is given by:
\begin{equation}\label{eq:hessian_formula}
\ddot{\cF}(\mu)    =
\int\left[ \nabla \psi^\top(z)\nabla_1 \nabla_2 \kpi(z,z')\nabla\psi(z')\right]d \mu(z')d\mu(z) + \int\left[ \nabla\psi^\top(z)H_1 \kpi(z,z')\nabla\psi(z)\right]d \mu(z')d\mu(z).
\end{equation}

\subsection{Proof of \Cref{prop:hessian_ksd_pi}}\label{sec:proof_hessian_ksd_pi}


Under \Cref{ass:square_der_int}, by the dominated convergence theorem, we can exchange the order of the integral and derivative in the second term of \eqref{eq:hessian_formula}:
\begin{align}\label{eq:scd_term_scd_der}
\iint\left[ \nabla\psi^\top(z)H_1 \kpi(z,z')\nabla\psi(z)\right]d \rho_t(z')d\rho_t(z) = \int \left[ \nabla\psi^\top(z) H_1\left(\int \kpi(z,z') d\rho_t(z')\right) \nabla\psi(z)\right] d\rho_t(z).
\end{align}
The latter \eqref{eq:scd_term_scd_der} vanishes when $\rho_t=\pi$, thanks to the property of the Stein kernel : $\int \kpi(z,z')\d \pi  (z')=0$.

Hence, by considering $\psi\in C_c^{\infty}(\X)$ and a path $\rho_t$ from  $\pi$ to $(I+\nabla\psi)_{\#}\pi$ given by: $\rho_t=  (I+t\nabla\psi)_{\#}\pi$, for all $t\in [0,1]$.,
\begin{equation*}
\ddot{\cF}(\pi)
= \Hess_{\pi}(\psi,\psi),
\end{equation*}
where
\begin{equation*}
\Hess_{\pi}(\psi,\psi)=
\iint\left[ \nabla \psi^\top(z)\nabla_1 \nabla_2 \kpi(z,z')\nabla \psi(z')\right]\d \pi(z')\d\pi(z).
\end{equation*}
By an integration by parts\footnote{First, differentiate the product $\frac{\partial \psi(x)}{\partial x_i}\pi(x)$, then integrate $\frac{\partial^2 \kpi(x,y)}{\partial x_i \partial y_j}$ w.r.t.\ $x_i$. Since $\psi$ is compactly supported, it vanishes at infinity. Since $\pi$ has a $C^1$-density w.r.t.\ the Lebesgue measure, any integral over $\d \pi$ bearing on the boundary of the support of $\pi$ vanishes. We thus do not have any extra integral on the boundary when performing the integration by parts.} w.r.t.\ $x$, we have, for $\cL_{\pi} : f\mapsto-\Delta f+\langle - \nabla \log \pi,\nabla f \rangle$,
\begin{align*}
\Hess_{\pi}&(\psi,\psi) 
= \sum_{i,j=1}^d \iint \frac{\partial \psi(x)}{\partial x_i}. \frac{\partial^2 \kpi(x,y)}{\partial x_i \partial y_j}. \frac{\partial \psi(y)}{\partial y_j}\d\pi(x)\d\pi(y)\\
&= - \sum_{i,j=1}^d \iint \frac{\partial^2 \psi(x)}{\partial x_i^2}. \frac{\partial \kpi(x,y)}{\partial y_j}. \frac{\partial \psi(y)}{\partial y_j}\d\pi(x)\d\pi(y) 
- \sum_{i,j=1}^d \iint \frac{\partial \psi (x)}{\partial x_i}. \frac{\partial \log \pi(x)}{\partial x_i}. \frac{\partial \kpi(x,y)}{\partial y_j}. \frac{\partial \psi(y)}{\partial y_j}\d\pi(x)\d\pi(y)\\
&= \sum_{j=1}^d\iint \cL_\pi\psi(x).\frac{\partial \kpi(x,y)}{\partial y_j}. \frac{\partial \psi(y)}{\partial y_j}\d\pi(x)\d\pi(y).
\end{align*}
We then repeat the same steps, performing an integration by parts w.r.t.\ $y_j$ and using \Cref{lem:kernel_operator_sp} in the last equality,
\begin{equation*}
   \Hess_{\pi}(\psi,\psi) = \int \cL_{\pi}\psi(x)\kpi(x,y)\cL_{\pi}\psi(y)\d\pi(x)\d\pi(y)= \| S_{\pi, \kpi}(\cL_{\pi}\psi)\|^2_{\Hkpi}.
\end{equation*}

Notice that under \Cref{ass:integrable}, we have $\Hkpi\subset L^2(\pi)$, so  $S_{\pi,\kpi}$ and its adjoint are well-defined (see \Cref{sec:background_ksd}).

\subsection{Proof of \Cref{prop:no_exp_cv}}\label{sec:no_exp_cv}

Denote by $L^2_0(\pi)$ the closed subspace of $L^2(\pi)$ consisting of functions $\phi$ such that $\int \phi(x)d\pi(x) = 0$. For any $f\in \mathcal{D}(\cL_{\pi})$, $\int \cL_{\pi}f(x)d\pi(x)=0$; hence $\text{Im}(\cL_{\pi})$ is a subset of  $L^2_0(\pi)$, which is itself a subset of $L^2(\pi)$.

Let $T_{\pi,\kpi}=S_{\pi,\kpi}^* \circ S_{\pi,\kpi}$, and assume that $V = -\log(\pi)$ is a $C^1(\R^d)$ function. We want to show that exponential decay near equilibrium \eqref{eq:exp_decay} holds, if and only if $\cL_{\pi}^{-1} : L_0^2(\pi) \rightarrow L_0^2(\pi)$, the inverse of $\cL_{\pi}\vert_{L_0^2(\pi)}$, is a well-defined linear operator, and
 \begin{equation} \label{eq:proof_prop_no_exp_cv_reformulation}
     \langle \phi , T_{\pi, \kpi} \phi \rangle_{L_2(\pi)} \geq \lambda \langle \phi , \cL_{\pi}^{-1} \phi \rangle_{L_2(\pi)}
 \end{equation}
holds for all $\phi \in L^2_0(\pi)$. 
\begin{proof}
First, assume that the exponential convergence near equilibrium holds. We will  show that $\cL_{\pi}^{-1}$ is a well-defined bounded linear operator on $L_0^2(\pi)$, and then that the inequality \ref{eq:proof_prop_no_exp_cv_reformulation} holds. Let $\psi \in C_c^{\infty}(\R^d)$. By \Cref{prop:hessian_ksd_pi}, the Hessian of $\cF$ at $\pi$ can be written as
\begin{equation*}
   \Hess_{\pi}(\psi,\psi) =  \ps{S_{\pi, \kpi}(\cL_{\pi}\psi),S_{\pi, \kpi}\cL_{\pi}\psi }_{\Hkpi}
= \ps{T_{\pi, \kpi}(\cL_{\pi}\psi),\cL_{\pi}\psi }_{L^2(\pi)},
\end{equation*}
where $T_{\pi,\kpi}= S_{\pi, \kpi}^* \circ S_{\pi, \kpi}$. By \Cref{lem:diffusion_background}, we have that $\|\nabla\psi \|^2_{L^2(\pi)}=\ps{\psi, \cL \psi}_{L^2(\pi)}$, and exponential convergence near equilibrium can be written:
\begin{equation}\label{eq:proof_propo_no_exp_cv_hess_equiv}
\ps{T_{\pi, \kpi}(\cL_{\pi}\psi),\cL_{\pi}\psi }_{L^2(\pi)}\ge \lambda \ps{\psi, \cL_{\pi} \psi}_{L^2(\pi)}.
\end{equation}
Note, that by \Cref{ass:integrable}, $T_{k_{\pi},\pi}$ is a bounded operator \citep[Theorem 4.27]{steinwart2008support}, and it follows by an application of the Cauchy-Schwartz inequality that \eqref{eq:proof_propo_no_exp_cv_hess_equiv} implies:
\begin{equation} \label{eq:poincare_inequality_extra_L}
    \Vert \cL_{\pi} \psi \Vert_{L^2(\pi)}^2 \geq \frac{\lambda}{\Vert T_{k_{\pi}, \pi} \Vert_{op}} \langle \psi, \cL_{\pi} \psi \rangle_{L^2(\pi)},
\end{equation}
where $\|\cdot\|_{op}$ denotes the operator norm. Now, let  $\psi$ be an arbitrary element of $\mathcal{D}(\cL_{\pi})$, the domain of $\cL_{\pi}$. By \Cref{lem:diffusion_background},  there exists a sequence $(\psi_n)_{n=1}^{\infty} \subset C_c^{\infty}(\R^d)$ converging strongly to $\psi$, such that $\cL_{\pi} \psi_n$ converges strongly to $\cL_{\pi}\psi$ as well. Hence, \eqref{eq:poincare_inequality_extra_L} holds for all $\psi \in \mathcal{D}(\cL_{\pi})$.

 We will now show, that the spectrum of $\cL_{\pi}$, $\sigma(\cL_{\pi})$ is contained in $\{0\} \cup [\frac{\lambda}{\Vert T_{k_{\pi}, \pi} \Vert_{op}}, \infty)$. Suppose that there exists a $\sigma \in (0, \frac{\lambda}{\Vert T_{k_{\pi}, \pi} \Vert_{op}}) \cap \sigma(\cL_{\pi})$. If $\sigma$ is in the point spectrum of $\cL_{\pi}$, then by definition there would exist a vector $v \in \mathcal{D}(\cL_{\pi})$ such that $\cL_{\pi} v = \sigma v$, which would contradict the inequality \eqref{eq:poincare_inequality_extra_L}. On the other hand, if $\sigma$ is not in the point spectrum of $\cL_{\pi}$, then by Weyl's criterion \citep[Theorem 7.22]{pankrashkin2014spectral} we can find an orthonormal sequence $(u_n)_{n=1}^{\infty}\in \mathcal{D}(\cL_{\pi})$ such that $(\cL_{\pi} - \sigma)u_n$ converges to $0$ in $L_2(\pi)$. An obvious calculation shows that this would contradict \eqref{eq:poincare_inequality_extra_L}. Hence, $\sigma(\cL_{\pi}) \subset \{0\} \cup [\frac{\lambda}{\Vert T_{k_{\pi}, \pi} \Vert}, \infty)$.
 
 We note that $L_0^{2}(\pi)$ is itself a Hilbert space with the inner product inherited from $L^2(\pi)$. The image of $\cL_{\pi}$ is contained in $L_0^2(\pi)$ (recall that for any $f\in \mathcal{D}(\cL_{\pi})$, $\int \cL_{\pi}f(x)d\pi(x)=0$), and it is dense in $L^{2}_0(\pi)$  since $\overline{\text{Im}(T)} = (\text{Ker}(T))^{\perp}$ for self-adjoint operators $T$ \citep[Corollary 2.18]{brezis2010functional}. Furthermore, since $\sigma(\cL_{\pi}) \subset \{0\} \cup [\frac{\lambda}{\Vert T_{k_{\pi}, \pi}\Vert_{op}} ,\infty) $ and the kernel of $\cL_{\pi}$ consists of constants (\Cref{sec:background_diffusion}), the operator $\tilde{\cL}_{\pi}=\cL_{\pi}|_{L^2_0(\pi)}$ (the restriction of $\cL_{\pi}$ to $L^2_0(\pi)$) is strictly positive and self-adjoint on $L^2_0(\pi)$. Since consequently $\sigma(\cL_{\pi}^{-1})\in (0,\frac{\|T_{\kpi,\pi\|_{op}}}{\lambda}]$, $\cL_{\pi}^{-1}: = \tilde{\cL}_{\pi}^{-1}$ is a well-defined, bounded,  self-adjoint and positive operator from $L_0^{2}(\pi)$ to $L_0^{2}(\pi)$. 
We now take an arbitrary $\psi \in \mathcal{D}(\cL_{\pi})$ and denote $\phi=\cL_{\pi}\psi$. We can write $\psi = \psi' + C$, where $C = \int \psi(x) d\pi(x)$ is a constant and $\psi' \in L_0^2(\pi) \cap \mathcal{D}(\cL_{\pi})$. We then have $\cL_{\pi} \psi = \cL_{\pi} \psi'$ and $\cL_{\pi}^{-1} \cL_{\pi} \psi = \psi' = \psi - C$. We also note, that since $\phi \in \text{Im}(\cL_{\pi})\subset L_0^2(\pi)$ we have $\langle C, \phi \rangle_{L^2(\pi)} = 0$.
Now by a direct substitution of $\phi = \cL_{\pi} \psi$ in \eqref{eq:proof_propo_no_exp_cv_hess_equiv} we obtain:
\begin{equation}\label{eq:exp_decay_equivalece}
\ps{T_{\pi, \kpi}\phi,\phi }_{L^2(\pi)}\ge  \lambda \ps{\psi'+C,\phi}_{L^2(\pi)}= \lambda \ps{\cL_{\pi}^{-1}\phi, \phi}_{L^2(\pi)}.
\end{equation}
for all $\phi$ in the image of $\cL_{\pi}$. Given that this image is dense in $L_0^{2}(\pi)$, and that $T_{k_{\pi}, \pi}$ and $\cL_{\pi}$ are continuous, this is equivalent to \eqref{eq:exp_decay_equivalece} holding for all $\phi \in L_0^{2}(\pi)$.

Now we prove the reverse implication, that is that if $\cL_{\pi}^{-1}$ is well-defined, bounded, and \eqref{eq:exp_decay_equivalece} holds for all $\phi \in L_0^{2}(\pi)$, then \eqref{eq:proof_propo_no_exp_cv_hess_equiv} holds for all $\psi \in C_c^{\infty}(\X)$. This follows trivially from the fact, that for $\psi \in C_c^{\infty}(\X)$ we have $\cL_{\pi} \psi \in L_0^{2}(\pi)$ and $\cL_{\pi}^{-1} \cL_{\pi} \psi = \psi - C$ where $C = \int \psi(x) d\pi(x)$ as previously. Again, using the fact that $\langle C, \cL_{\pi} \psi \rangle_{L^2(\pi)} = 0$ and a direct substitution $\phi = \cL_{\pi} \psi$ into \eqref{eq:exp_decay_equivalece}, we obtain \eqref{eq:proof_propo_no_exp_cv_hess_equiv} for all $\psi \in C_c^{\infty}(\R^d)$.
\end{proof}

\subsection{Proof of \Cref{coro:spectr_decay}\label{sec:proof_eigenvalue_decay} }
We begin this section by an additional result, that shows that the assumption that the spectrum of $\cL_{\pi}^{-1}$ is necessarily discrete  if exponential decay near equilibrium \eqref{eq:exp_decay} holds and the RKHS of $\kpi$ is infinite dimensional.

\begin{lemma}\label{lem:discrete_spectrum_hkpi}
If exponential decay near equilibrium \eqref{eq:exp_decay} holds, and the RKHS for $\kpi$ is infinite dimensional, then $\cL_{\pi}^{-1}$ has a discrete spectrum.
\end{lemma}

\begin{proof}
By \Cref{prop:no_exp_cv}, exponential convergence near equilibrium implies that $\cL_{\pi}^{-1}$ is a well-defined bounded linear operator on $L_0^2(\pi)$ and the inequality \eqref{eq:proof_prop_no_exp_cv_reformulation_main} holds for all $\phi \in L_2^0(\pi)$.
Let $(l_n)_{n \in \mathbb{N} }$ be the eigenvalues of $T_{k, \pi_k}$ in descending order. Using the max-min variational formula for the eigenvalues of a compact operator, and applying \Cref{prop:no_exp_cv}, we get:
\begin{equation} \label{eq:rayleigh_quotients_eigenvalue_decay}
    l_n = \sup_{\substack{E \subset L^2_0(\pi) \\ dim(E) = n}} \inf_{\substack{\phi \in E \\ \phi \neq 0}} \frac{\langle \phi, T_{k, \pi_k} \phi \rangle_{L^2(\pi)}}{\langle \phi, \phi \rangle_{L^2(\pi)}} \geq \lambda \sup_{\substack{E \subset L_{0}^2(\pi)  \\ dim(E) = n}} \inf_{\substack{\phi \in E \\ \phi \neq 0}} \frac{\langle \phi, \cL_{\pi}^{-1} \phi \rangle_{L^2(\pi)}}{\langle \phi, \phi \rangle_{L^2(\pi)}} .
\end{equation}
We will now show, that this implies that the spectrum of $\cL_{\pi}^{-1}$ is discrete. Let $\cL_{\pi}^{1/2}$ be the square root of $\cL_{\pi} \vert_{L_0^2(\pi)}$, the restriction of $\cL_{\pi}$ to $L^2_0(\pi)$. As a consequence of the analysis in \Cref{sec:no_exp_cv}, the operator $\cL_{\pi}^{1/2}$ is well-defined, strictly positive  and self-adjoint from $\mathcal{D}(\cL_{\pi}) \cap L_{0}^2(\pi)$ to $L_{0}^2(\pi)$. By retricting the supremum to $\phi$ of the form $\phi = \cL_{\pi}^{1/2} \psi$, where $\psi \in \mathcal{D}(\cL_{\pi}) \cap L_0^2(\pi)$, we obtain the following lower bound:
\begin{align} \label{eq:rayleigh_quotients_eigenvalue_decay_2}
    \sup_{\substack{E \subset L_{0}^2(\pi)  \\ dim(E) = n}} \inf_{\substack{\phi \in E \\ \phi \neq 0}} \frac{\langle \phi, \cL_{\pi}^{-1} \phi \rangle_{L^2(\pi)}}{\langle \phi, \phi \rangle_{L^2(\pi)}}  & \geq \sup_{\substack{E \subset \cL_{\pi}^{1/2}(L_{0}^2(\pi) \cap \mathcal{D}(\cL_{\pi}))  \\ dim(E) = n}} \inf_{\substack{\phi \in E \\ \phi \neq 0}} \frac{\langle \phi, \cL_{\pi}^{-1} \phi \rangle_{L^2(\pi)}}{\langle \phi, \phi \rangle_{L^2(\pi)}}  \\
    & = \sup_{\substack{E \subset L_{0}^2(\pi) \cap \mathcal{D}(\cL_{\pi})  \\ dim(E) = n}} \inf_{\substack{\psi \in E \\ \psi \neq 0}} \frac{\langle \psi, \psi \rangle_{L^2(\pi)}}{\langle \psi, \cL_{\pi} \psi \rangle_{L^2(\pi)}} = \left( \inf_{\substack{E \subset L_{0}^2(\pi) \cap \mathcal{D}(\cL_{\pi})  \\ dim(E) = n}} \sup_{\substack{\psi \in E \\ \psi \neq 0}} \frac{\langle \psi, \cL_{\pi} \psi \rangle_{L^2(\pi)}}{\langle \psi, \psi \rangle_{L^2(\pi)}} \right)^{-1} \nonumber
\end{align}
It now follows that $\cL_{\pi} \vert_{L_0^2(\pi)}$ has discrete spectrum. Indeed, if it was not the case, by \citep[Theorem 8.1]{pankrashkin2014spectral} the quotients within the parenthesis on the right hand side of \eqref{eq:rayleigh_quotients_eigenvalue_decay_2}, would be upper bounded by $\inf \sigma_{\text{ess} }(\cL_{\pi} \vert_{L_0^2(\pi)}) > 0$, where $\sigma_{\text{ess}}$ is the essential spectrum, and thus the right hand side of \eqref{eq:rayleigh_quotients_eigenvalue_decay_2} would be lower bounded by a constant equal to $\left(\inf \sigma_{\text{ess} }(\cL_{\pi} \vert_{L_0^2(\pi)}) \right)^{-1} > 0$. This is not possible since the eigenvalues $(l_n)_{n\in \mathbb{N}}$ converge to zero. Since $\cL_{\pi}$ is positive and unbounded on $L_0^2(\pi)$ and its spectrum is purely discrete, it follows that $\cL_{\pi}^{-1}$ is compact and, by extension, $(\cL_{\pi} + I)^{-1}: L^2(\pi) \rightarrow L^2(\pi)$ is compact, which means that $\cL_{\pi}$ has compact resolvent. \ak{the end is not clear - why don't we stop at lpi-1 has discrete spectrum?}\ak{do you want to show that lpi has discrete spectrum or lpi-1? otherwise change the statement of Lemma 16. to me we want lpi-1}
\end{proof}

The proof of \Cref{coro:spectr_decay} now follows readily:
\begin{proof}[Proof of \Cref{coro:spectr_decay}]
We denote by $(l_n)_{n \in \mathbb{N} }$ the eigenvalues of $T_{k, \pi_k}$ in descending order and by $(\lambda_n)_{n \in \mathbb{N}}$ the eigenvalues of compact operator $\cL_{\pi}^{-1} : L_0^2(\pi) \rightarrow L_0^2(\pi)$ also in descending order. By max-min variational characterization of eigenvalues for compact operators, it follows from \eqref{eq:rayleigh_quotients_eigenvalue_decay} that $l_n \geq \lambda. \lambda_n$, hence $\lambda_n = \mathcal{O}(l_n)$.
\end{proof}

\subsection{Bound on eigenvalue decay and proof of \Cref{cor:no_exp_decay} }\label{sec:eigenvalue_decay_proof}
\begin{lemma} \label{lem:eigenvalue_decay}
	Let $\gamma_d \sim \mathcal{N}(0, I_d)$
	be the standard $d$-dimensional Gaussian measure, and let $\cL_{\gamma_d} = -\Delta +  \ps{x , \nabla}$ be the Ornstein-Uhlenbeck operator on $L^2(\gamma_d)$. If we denote by $\rho_n$ the $n$-th smallest eigenvalue of $\cL_{\gamma_d}$, then we have:
	\begin{equation*}
	\rho_n = \mathcal{O}(n^{1/d}).
	\end{equation*}
\end{lemma}
\begin{proof}
	For a multi-index $\alpha = (k_1, \ldots, k_d)$, the multivariate Hermite polynomial is for any $x=(x_1,\dots,x_d)\in \R^d$:
	$$ H_{\alpha}(x) = \prod_{i=1}^d H_{k_i}(x_i) $$
	where for $k_i \in \mathbb{N}$, $H_{k_i}$ denotes the usual one-dimensional $k_i$th-order Hermite polynomial.
	It is well known that multivariate Hermite polynomials form an orthogonal basis of $L^2(\gamma_d)$ and that we have $$\cL_{\gamma_d} H_{\alpha} = \vert \alpha \vert H_{\alpha}$$ where $\vert \alpha \vert = \sum_{i=1}^d k_i$ \citep[Section 2.7.1]{bakry2013analysis}.
	Therefore any $k \in \mathbb{N}$ is an eigenvalue of $\cL_{\gamma_d}$, with multiplicity equal to
	$$S_k = {{k + d -1} \choose {d-1}}$$
	which is the number of solutions of the equation $k = \sum_{i=1}^d k_i$, where $\{k_i\}_{i=1}^d$ takes its values in the set of nonnegative integers. This means, that \begin{equation*}
	\rho_n = k \hspace{4pt} \Longleftrightarrow \hspace{4pt} \sum_{i=1}^{k-1} S_i < n \leq \sum_{i=1}^{k} S_i.
	\end{equation*}
	Since $S_i$ is a polynomial in $i$ of degree $d-1$, then $\sum_{i=1}^k S_i$ is a polynomial in $k$ of degree $d$. 
	Therefore we have $\rho_n = \mathcal{O}(n^{1/d})$.
\end{proof}
\begin{corollary} \label{cor:eigenvalue_decay}
	The conclusion of \Cref{lem:eigenvalue_decay} holds for any Schrödinger operator on $L_2(\R^d)$, defined for $L > 0$ as follows,
	\begin{equation}\label{eq:gaussian_Schrödinger}
	\Sop_{\nu_L} := - \Delta + \frac{1}{4} L^2 \Vert x\Vert^2 - \frac{1}{2}dL.
	\end{equation}
\end{corollary}
\begin{proof}
	Let $\nu_L \sim \mathcal{N}(0, \frac{1}{L}I_d )$ be a normal measure, and let $\cL_{\nu_L} = - \Delta + L \ps{x , \nabla}$ be the associated Ornstein-Uhlenbeck operator on $L_2(\nu_L)$. It is easy to see that the map $R_L : L_2(\gamma_d) \rightarrow L_2(\nu_L)$ given by:
	\[
	(R_L \phi)(x) = \phi(\sqrt{L} x)  
	\]
	is unitary, and that $\cL_{\gamma_d} = R_L^* \cL_{\nu_L} R_L$. Furthermore, it is standard that $\cL_{\nu_L}$ is unitarily equivalent to $\Sop_{\nu_L}$, see \citep[Proposition 4.7]{pavliotis2014stochastic}.
	Since unitary equivalence preserves the spectrum, the thesis follows.
\end{proof}
\begin{lemma} \label{lem:main_lemma_no_exp_decay}
	Suppose that $\pi \propto e^{-V}$ where $V$ is a $C^2(\R^d)$ potential such that $\nabla V$ is $L$-Lipschitz and assume that $\cL_{\pi}$ has discrete spectrum. If we denote by $\tilde{\lambda}_n$ the $n$-th smallest eigenvalue of the operator $\cL_{\pi}$ (counting the multiplicities), then 
	\begin{equation*}
	\tilde{\lambda}_n \leq \mathcal{O}(n^{1/d})
	\end{equation*}
\end{lemma}
\begin{proof}
	It is easy to show, that if $\nabla V$ is Lipschitz and $\int e^{-V(x)}\d x  $ is finite, then $\lim_{|x| \rightarrow \infty} V(x)  = \infty$. This is a consequence of the fact that $\lim_{R \rightarrow \infty} \int_{|x| > R} e^{-V(x)} \d x   = 0$ and that on the set $\{y : ||x - y|| \leq 1, \langle \nabla V(x), y -x \rangle \leq 0 \}$ we have $ - V(y) \geq -V(x) - \frac{L}{2} \Vert x-y\Vert_2^2$. Therefore $V(x)$ has to diverge to $\infty$ as $|x| \rightarrow \infty$. It follows, that $V$ attains a minimum on $\R^d$. Assume, without loss of generality, that $V$ attains its (not necessarily unique) minimum at $0$. The operator $\cL_{\pi}$ is unitarily equivalent to the Schrödinger operator (Equation 4.130, \citet{pavliotis2014stochastic}) 
	\begin{equation*}
	\Sop_{\pi} = - \Delta + \frac{1}{4}\Vert \nabla V \Vert^2 - \frac{1}{2} \Delta V
	\end{equation*}
	on $L_2(\R^d)$. Since unitary equivalence preserves the spectrum, this operator has by assumption a discrete spectrum.
	We will show that the eigenvalues of $\Sop_{\nu_L} + dL I$ dominate those of $\Sop_{\pi}$ where $\Sop_{\nu_L}$ is defined in \eqref{eq:gaussian_Schrödinger}. For any $\phi \in D(\Sop_L)$, we have
	\begin{align*}
	+\infty > \langle \phi, \Sop_{\nu_L} + dLI \phi \rangle_{L_2(\R^d)} & = \langle \phi, \Sop_{\pi} \phi \rangle_{L_2(\R^d)} + \frac{1}{4}\int \left(L^2 \Vert x\Vert^2 - \Vert \nabla V(x) \Vert^2 + 2dL + 2\Delta V(x) \right) \phi(x)^2 \d x   \\
	& \geq \langle \phi, \Sop_{\pi} \phi \rangle_{L_2(\R^d)},
	\end{align*}
	where the last inequality follows from the $L$-Lipschitz regularity of $\nabla V$. 
	It follows from the above calculation that the domain of the quadratic form $\langle \phi, \Sop_{\nu_L} + dLI \rangle_{L_2(\R^d)}$, denoted $Q(\Sop_{\nu_L} + dLI)$ is contained in the domain of the quadratic form $\langle \phi, \Sop_{\pi} \phi \rangle_{L_2(\R^d)}$, denoted $Q(\Sop_{\pi})$. It is now a simple consequence of the the Rayleigh-Ritz variational formula that the eigenvalues of $\Sop_{\nu_L} + dLI$ dominate those of $\Sop_{\pi}$ \citep[Corollary 8.6]{pankrashkin2014spectral}, that is, $\tilde{\lambda}_n \leq \rho_n + dL$
	where $\rho_n$ is the $n$-th smallest eigenvalue of the operator $\Sop_{\nu_L}$. By \Cref{cor:eigenvalue_decay}, we have that $\rho_n = \mathcal{O}(n^{1/d})$, which concludes the proof.
\end{proof}

\begin{proof}[Proof of \Cref{cor:no_exp_decay}]

Let $(l_n)_{n=1}^{\infty}$ be the eigenvalues of $T_{k_{\pi}, \pi}$ and $(\lambda_n)_{n=1}^{\infty}$ be the eigenvalues of $\cL_{\pi}^{-1}$ both in descending order. Since $T_{k_{\pi}, \pi}$ is a trace class operator \citep[Theorem 4.27]{steinwart2008support}, we know that $\sum_{n=1}^{\infty} l_n < + \infty$. By \Cref{cor:eigenvalue_decay} exponential convergence near equilibrium would imply that  $\sum_{n=1}^{\infty} \lambda_n < + \infty $ as well. On the other hand we have that $\lambda_n = \tilde{\lambda_n}^{-1}$, where $\tilde{\lambda}_n$ are the positive eigenvalues of $\cL_{\pi}$ in a nondecreasing order. By \Cref{lem:main_lemma_no_exp_decay} we have $\tilde{\lambda}_n = \mathcal{O}(n^{1/d})$, which implies that there exists an $N \in \mathbb{N}$ and a constant $C > 0$, such that for all $n \geq N$ we have $\tilde{\lambda}_n \leq Cn^{1/d}$ and thus $\lambda_n \geq Cn^{-1/d}$. This contradicts the summability of $\lambda_n$, and concludes the proof.
\end{proof}

\subsection{Proof of \Cref{lem:no_cst}}\label{sec:proof_no_cst}

Let $\cH_0=\left\{  \sum_{i=1}^{m}
\alpha_i \kpi(\cdot, x_i); m \in \mathbb{N}; \alpha_1, \dots, \alpha_m \in \R; x_1, \dots, x_m \in \X
\right\}$.
Recall that $\Hkpi$ is the set of functions on $\X$ for which there exists an
$\cH_0$-Cauchy sequence $(f_n)_n\in \cH_0^\N$ converging pointwise to $f$. Let $c\in \R$.
Let $f_0=\sum_{i=1}^{m}
\alpha_i \kpi(\cdot, x_i)\in \cH_0$ and assume that $f_0=c$. Integrating $f_0$ w.r.t.\ $\pi$ yields 
\begin{equation*}
c=\int f_0(x)d\pi(x)=\sum_{i=1}^{m}\alpha_i \int \kpi(x_i,x)d\pi(x)=0.
\end{equation*}
Hence $c=0$. Similarly, let $f\in \Hkpi$ such that $f=c$, fix $(f_n)_n\in \cH_0^\N$ such that $\|f_n- f\|_{\Hkpi}\to 0$.

Under \Cref{ass:integrable}, the injection $\iota:\Hkpi\to L^2(\pi)$ is continuous, linear and bounded. Indeed, for any $g\in \Hkpi$
\begin{equation*}
\| \iota g\|^2_{L^2(\pi)}=\int g(x)^2 d\pi(x) = \int \ps{g,\kpi(x,\cdot)}^2_{\Hkpi} d\pi(x) \le \|g\|^2_{\Hkpi}\int \kpi(x,x)d\pi(x)=:c_{\pi}^2\|g\|^2_{\Hkpi},
\end{equation*}
where $c_\pi<+\infty$ exists by \Cref{ass:integrable}. So $\|f_n- f\|_{\Hkpi}\to 0$ implies that $\|f_n-f\|_{L^2(\pi)}\to 0$. Since $\int f_n(x)d\pi(x)=0$, by the reproducing property and Cauchy-Schwartz inequality, we have:
\begin{equation*}
|c|= \left|\int f(x)d\pi(x)\right|= \left|\int (f_n-f)(x)d\pi(x)\right|\le \int | \ps{f_n-f,\kpi(x,\cdot)}_{\Hkpi}|d\pi(x)\le \|f_n-f\|_{\Hkpi} c_{\pi}^2.
\end{equation*}
which shows that $c=0$ since $\|f_n- f\|_{\Hkpi}\to 0$.

\subsection{Proof of \Cref{prop:stable support}}\label{sec:proof_stable_support}

Rather than providing a proof only for the more restrictive smooth submanifolds considered in \Cref{prop:stable support}, we express the result below for general closed nonempty sets. This formulation involves contingent cones which are crucial quantities for studying the invariance of non-smooth sets. They can be informally described as the collection of directions at $x$ that point either inward or that are tangent to the set $\cM$. More formally, for $d_\cM(y)$ the distance of $y\in\X$ to $\cM$, $T_\cM(x):=\{v|\liminf_{h\rightarrow 0^+} d_\cM(x+h v)/h=0\}$. Non-smooth sets were experimentally met whenever we considered Gaussian mixtures with more than three components for which there was no simple axis of symmetry (as on \Cref{fig:skewed_mog}). These are cases more intricate than the ones considered in \Cref{lem:Flow-invariant symmetry}.

\begin{proposition}\label{prop:stable_support_nonsmooth}
	Let $\cM\subset \R^d$ be a closed nonempty set and $\mu_0\in P_c(\X)$ with $\supp(\mu_0)\subset\cM$. Assume that, for a deterministic $(v_{\mu_t})_{t\ge 0}$ satisfying the Caratheodory-Lipschitz Assumptions \ref{ass:caratheodory_field_contr}-\ref{ass:lip_field_wrtMeasure_contr}, we have $v_{\mu_t}(x)\in T_\cM(x)$ where $ T_\cM(x)$ is the contingent cone of $\cM$ at $x\in \cM$. Then $\cM$ is flow-invariant for \eqref{eq:continuity_eq}.
\end{proposition}
\begin{proof} Consider any $x_0\in\supp(\mu_0)$. By \Cref{ass:sublinear_growth_contr} and Gronwall's lemma, $x'(t)=v_{\mu_t}(x(t))$ can only generate trajectories that do not explode in finite time, i.e.\ there is no $T<\infty$ such that $\limsup_{t\rightarrow T^-} \|x(t)\|=+\infty$. Since $v(\mu)(\cdot)$ is continuous by \Cref{ass:caratheodory_field_contr}, and $v_{\mu_t}(x)\in T_\cM(x)$, we can apply an invariance result \citep[Theorem 10.4.1]{Aubin1990SVA} stating that all the generated trajectories $x(\cdot)$ stay within $\cM$ at all times. This can be informally understood as using directions that are always tangent or pointing within $\cM$ cannot push $x(\cdot)$ outside of $\cM$.

The superposition principle \citep[Theorem 3.4]{Ambrosio2014Superp} states that $\supp(\mu_t)$ is contained in the set of positions $x(t)$ reached by all the trajectories satisfying $x'(t)=v_{\mu_t}(x(t))$ for some $x_0=x(0)\in\supp(\mu_0)$. Consequently $\supp(\mu_t)\subset \cM$, namely $\cM$ is flow-invariant for \eqref{eq:continuity_eq}.
\end{proof}

\subsection{Proof of \Cref{lem:Flow-invariant symmetry}}\label{sec:proof_flow_inv_sym}

Assume that $d\pi(x) \propto e^{-V(x)}d\lambda(x)$, then $s(x)=-\nabla V(x)$. Translating and reindexing, w.l.o.g.\ we can take $\cM=\Sp(e_I,\dots,e_d)$. Thus by symmetry of $\pi$ w.r.t.\ $\cM$, we have that $s_i(x)=(\nabla \log \pi(x))_i=0$ for all $i< I$ and $x\in\cM$, so $s(x)\in\cM$. Hence, for every $x\in\cM$, $\partial_j s_i(x)=0$, for $i< I$ and $I \le j\le d$. Therefore, for any $u\in \cM$ $[J(s)(x)] u=[J(s)(x)]_{|\cM}u$, i.e.\ $\cM$ is left invariant under the action of any Jacobian matrix $[J(s)(x)]$ for $x\in\cM$.

Consider a radial kernel $k(x,y)=\phi(\|x-y\|^2/2)$ with $\phi\in C^3(\R)$. Then $\nabla_2 k(x,y)=(x-y)\phi'(\|x-y\|^2/2)$ and $H_{1}k(x,y)=\phi'(\|x-y\|^2/2) I_d + (x-y)\otimes (x-y) \phi''(\|x-y\|^2/2)$. As, for all $x,y\in\cM$, $x-y\in\cM$, $\nabla_2 k(x,y)\in\cM$ and $\cM$ is left invariant under the action of $H_{1}k(x,y)$. Since $\div_{1}\nabla_{2}k(x,y)$ is radial as well, $\nabla_2 ( \div_{1}\nabla_{2}k(x,y))\in \mathrm{span}(x-y)$. We have thus shown that all the terms in \eqref{eq:nabla2_kpi_2} belong to $\cM$ if $x,y\in\cM$, so $\nabla_2 \kpi(x,y)\in\cM$.

Consequently $\nabla_{W_2}\cF(\mu)(y)=\mathbb{E}_{x \sim \mu}[\nabla_{2}\kpi(x,y)]\in \cM$ for any $\mu$ such that $\supp(\mu)\subset \cM$. Since the contingent cone of a subspace at any point is equal to the subspace itself, we conclude by applying \Cref{prop:stable_support_nonsmooth} to $\cM$.








\subsection{Proof of \Cref{lem:diffusion_background}} \label{sec:proof_diffusion_background}

We start by noting, that $\tilde{\cL}_{\pi}(\phi) = \nabla V \cdot \nabla \phi$ is well-defined on $C_c^{\infty}(\R^d)$. Denote $Z = \int_{\R^d} e^{-V(x)}\d x  $. Using integration by parts, we have for $\phi,\psi \in C_c^{\infty}(\R^d)$:
\begin{align*}
    \langle \tilde{\cL}_{\pi} \phi, \psi \rangle_{L_2(\pi)} &=  \frac{1}{Z}\int_{\R^d} (\nabla V \cdot \nabla \phi(x) - \Delta \phi(x))\psi(x) e^{-V(x)} \d x   \\ &= \frac{1}{Z} \int_{\R^d} (\nabla V \cdot \nabla \phi(x)) \psi(x) e^{-V(x)} + \nabla \phi(x) \cdot (\nabla \psi(x) e^{-V(x)}) \d x   \\
    & = \frac{1}{Z} \int \langle \nabla \phi(x), \nabla \psi(x) \rangle e^{-V(x)} \d x   = \langle \nabla \phi, \nabla \psi \rangle_{L_2(\pi)}.
\end{align*}
It follows that $\tilde{\cL}_{\pi}$ is symmetric, and since $\langle \tilde{\cL}_{\pi} f,  f \rangle = \Vert \nabla f \Vert^2_{L_2(\pi)}$, it is positive as well. We can now define $\cL_{\pi}$ as the Friedriechs extension of $\tilde{\cL}_{\pi}$ over $L_2(\pi)$ \citep[Definition 2.17]{pankrashkin2014spectral}. This means that, when we consider the bilinear form $F(\phi,\psi) = \langle \tilde{\cL}_{\pi} \phi, \psi\rangle$, there exists the smallest closed bilinear form $\overline{F}$ on $L_2(\pi)$ which extends it \citep[Proposition 2.16]{pankrashkin2014spectral}. The operator associated with $\overline{F}$ is a self-adjoint extension of $\tilde{\cL}_{\pi}$ on $L_2(\pi)$, which is also positive. We denote this extension by $\cL_{\pi}$. Furthermore, the domain of $\cL_{\pi}$ is by definition contained in the domain of $\overline{F}$, which is the closure of $C_c^{\infty}$ for the norm \eqref{eq:bilinear_form_norm} \citep[Proposition 2.8]{pankrashkin2014spectral}. The claim of density of $C_c^{\infty}(\R^d)$ in $\mathcal{D}(\cL_{\pi})$ for the norm \eqref{eq:bilinear_form_norm} now follows.

Recall, that we have shown that for $\phi \in C_c^{\infty}(\R^d)$ we have $\langle \phi , \cL_{\pi} \phi \rangle = \Vert \nabla \phi \Vert_{L^2(\pi)}^2$. Therefore $$\Vert \phi \Vert_{\cL_{\pi} }^2 = \left( \Vert \nabla \phi \Vert_{L^2(\pi)}^2 + \Vert \phi \Vert_{L^2(\pi)}^2 \right)$$ which is the $W^{1,2}(\pi)$ Sobolev norm. This means, that the domain of the closure $\overline{F}$ is equal to $W_0^{1,2}(\pi)$, and hence $\mathcal{D}(\cL_{\pi}) \subset W_0^{1,2}(\pi)$. In fact, one can establish that $\mathcal{D}(\cL_{\pi}) = W_0^{1,2}(\pi) = W^{1,2}(\pi)$, though we will not need this. It now follows that for any $f \in \mathcal{D}(\cL_{\pi})$ there exists a weak derivative $\nabla f \in L_2(\pi)$, and it is easy to establish the equality $\Vert \nabla f \Vert_{L_2(\pi)}^2 = \langle f ,\cL_{\pi}f \rangle_{L_2(\pi)}$ by approximating $f$ in the norm $\Vert \cdot \Vert_{W^{1,2}(\pi)}$. 


It is now easy to show, that the kernel of $\cL_{\pi}$ consists of constant functions. If $\cL_{\pi} f = 0$, we have $\Vert \nabla f \Vert_{L^2(\pi)}^2 = \langle f, \cL_{\pi} f \rangle_{L^2(\pi)} = 0$.  One can deduce, that if the weak derivative is zero then the function is constant by a standard argument using mollifiers that we sketch below:

Define the convolution in $L^2(\pi)$ by $f \star g (x)= \int f(x-y) g(y) d\pi(y)$. By standard properties of the convolution, for any $\phi \in C_c^{\infty}$, we have $f \star \phi \in C^{\infty} \cap L^2(\pi)$ and $\nabla (f \star \phi) = f \star (\nabla \phi)$, and if $f$ has a weak derivative then by definition $f \star (\nabla \phi) = (\nabla f) \star \phi$. Let $\phi_n$ be now a  mollifier then again, by standard arguments, $f \star \phi_n$ converges to $f$ in $L_2(\pi)$, and also $f \star\phi_n$ is smooth with $\nabla (f \star \phi_n) = (\nabla f ) \star \phi_n = 0$, hence $f \star \phi_n$ is constant. It follows that  $f$ is constant.

We are left to show that the set $\{ (\phi, \cL_{\pi} \phi) \in L^2(\pi) \times L^2(\pi) : \phi \in C_c^{\infty}(\R^d) \}$ is dense in the graph of $L^2(\pi)$ for the topology inherited from the normed space $L^2(\pi) \times L^2(\pi)$. The domain of the closed form $\overline{F}$ with the norm \eqref{eq:bilinear_form_norm} is a Hilbert space \citep[Definition 2.5]{pankrashkin2014spectral}, with an inner product denoted by $\langle  \cdot , \cdot \rangle_{\cL_{\pi} }$. We have that for any $f \in \mathcal{D}(\cL_{\pi})$ there exists a sequence $(\phi_n)_{n=1}^{\infty}$ such that $\lim_{n \rightarrow \infty} \Vert \phi_n - f \Vert_{\cL_{\pi}} = 0$. By the Cauchy-Schwarz inequality, we have that for any $\psi \in C_c^{\infty}$:
\begin{equation*}
    \vert \langle \psi , \cL_{\pi} (f - \phi_n) \rangle_{L^2(\pi)} + \langle \psi , f - \phi_n \rangle_{L^2(\pi)} \vert = \vert \langle \psi, f - \phi_n \rangle_{\cL_{\pi}} \vert \leq \Vert \psi \Vert_{\cL_{\pi}} \Vert f - \phi_n \Vert_{\cL_{\pi}},
\end{equation*}
and since $\phi_n$ converges to $f$ in $L^2(\pi)$ strongly, we obtain $ \lim_{n \rightarrow \infty} \langle \psi , \cL_{\pi} (f - \phi_n) \rangle = 0$ for all $\psi \in C_c^{\infty}(\R^d)$. Since $C_c^{\infty}(\R^d)$ is dense in $L^2(\pi)$, it follows that $\cL_{\pi} \phi_n$ converges weakly to $\cL_{\pi} f$. By a version of Mazur's lemma on weakly and strongly closed convex sets \citep[Lemma 10.19]{renardy2006introduction}, it follows that there exists a sequence of finite convex combinations of $\cL_{\pi} \phi_n$ converging strongly to $\cL_{\pi} f$. More precisely, there exists a function $N : \mathbb{N} \rightarrow \mathbb{N}$ and a sequence of sets of real numbers $\{ \alpha_{k,m}\}_{m=k}^{N(k)}$ such that $\alpha_{k,m} \geq 0$, $\sum_{m=k}^{N(k)} \alpha_{k,m} = 1$, and 
$$
\lim_{k \rightarrow \infty} \sum_{n = k}^{N(k)} \alpha_{k,n}\cL_{\pi} \phi_n = \cL_{\pi} f  
$$
in the strong topology on $L^2(\pi)$. It follows by linearity of $\cL_{\pi}$, that the functions $\psi_k = \sum_{n=k}^{N(k)} \alpha_{k,n} \phi_n$ are $C_c^{\infty}$ functions that converge to $f$ in $L^2(\pi)$ strongly, and $\cL_{\pi} \psi_k$ converges to $\cL_{\pi} f$ strongly in $L^2(\pi)$ as well.

\subsection{Proof of \Cref{prop:descent_lemma_appendix}}\label{sec:descent_lemma_proof}

To perform the computations related to \Cref{lem:derivative_ksd} for the induction formula \eqref{eq:ksd_descent}, we need a compactly-supported push-forward, however $\nabla_{W_2}\cF(\mu_n)(\cdot)$ is not compactly supported in general. We consequently leverage the compactness of the measure it is applied on to perform our analysis. We thus first show by induction that, owing to \eqref{eq:ksd_descent}, $\mu_n\in\cP_c(\R^d)$ for all $n\in\N$, since $\mu_0\in\cP_c(\R^d)$. Assume that $\mu_n\in\cP_c(\R^d)$, then, by definition of the push-forward operation, since $\nabla_{W_2}\cF(\mu_n)\in C^1(\X;\X)$,
\begin{equation}
    \supp(\mu_{n+1})\subset \left(I-\gamma \nabla_{W_2}\cF(\mu_n)\right) (\supp(\mu_{n}))\subset\supp(\mu_{n}) + \gamma \left(\max_{x\in\supp(\mu_{n})}\|\nabla_{W_2}\cF(\mu_n)(x)\|_2\right) B(0,1)=:S_n.
\end{equation}
Hence $\mu_{n+1}\in\cP_c(\R^d)$ as claimed.

Consider a path between $\mu_{n}$ and $\mu_{n+1}$ of the form $\rho_{t}=\left(I-\gamma t \nabla_{W_2}\cF(\mu_n)\right)_{\#} \mu_{n}$. Set $\phi(x)=-\gamma\nabla_{W_2}\cF(\mu_n)(x)$, and
 $s_{t}(x)=x+t\phi(x)$ which is distributed according to $\rho_{t}$ for $x$ distributed as $\mu_n$. In general the function $\left(I-\gamma \nabla_{W_2}\cF(\mu_n)\right)$ is not compactly supported so we cannot apply \Cref{lem:derivative_ksd} outright. But since the push-forward is only applied to the compactly supported $\mu_{n}$ in the definition of $\rho_t$, we can find, through a mollifier, a function $f_{n,t}\in C^1(\X;\X)$ such that it coincides with $\left(I-\gamma t \nabla_{W_2}\cF(\mu_n)\right)$ on $S_n$ and is supported on $S_n+B(0,1)$. So $\rho_t=(f_{n,t})_{\#}\mu_{n}$ and we can apply \Cref{lem:derivative_ksd}. Hence $t\mapsto \cF(\rho_t)$ is differentiable and absolutely continuous, and consequently
\begin{equation}\label{eq:develop}
\cF\left(\rho_{1}\right)-\cF\left(\rho_{0}\right)=\dcF\left(\rho_{0}\right)+\int_{0}^{1} \left[\dcF\left(\rho_{t}\right)-\dcF\left(\rho_{0}\right)\right] d t,
\end{equation}
where 
$$\dcF\left(\rho_{t}\right)=\mathbb{E}_{\substack{x\sim \mu_n \\ x'\sim \mu_n}}\left[\nabla_{1} \kpi\left(s_{t}(x), s_{t}(x')\right)^\top \left(\phi(x)\right)\right].$$
The first term in the r.h.s. of \eqref{eq:develop} is
 $$\dot{\cF}(\rho_0)=-\gamma\E_{x\sim \mu_n}\left[ \| 
\nabla_{W_2}\cF(\mu_n)(x)
\|^2\right].$$
Since $s_{t}(x)-s_{t'}(x)=\left(t-t'\right) \phi(x)$, by \Cref{ass:lipschitz}, we derive through Jensen's and Cauchy-Schwarz inequalities that
\begin{align*}
&\left|\dcF\left(\rho_{t}\right)-\dcF\left(\rho_{t'}\right)\right|  = \left|\mathbb{E}_{\substack{x\sim \mu_n \\ x'\sim \mu_n}}\left[\left(\nabla_{1} \kpi\left(s_{t}(x), s_{t}(x')\right)-\nabla_{1} \kpi\left(s_{t'}(x), s_{t'}(x')\right)\right)^\top \phi(x)\right]\right|\\
&\le \mathbb{E}_{\substack{x\sim \mu_n \\ x'\sim \mu_n}}\left[\left(\left\|\nabla_{1} \kpi\left(s_{t}(x), s_{t}(x')\right)-\nabla_{1} \kpi\left(s_{t'}(x), s_{t}(x')\right)\right\| +\left\|\nabla_{1} \kpi\left(s_{t'}(x), s_{t}(x')\right)-\nabla_{1} \kpi\left(s_{t'}(x), s_{t'}(x')\right)\right\|\right) \left\|\phi(x)\right\|\right]\\
&\le |t-t'|\mathbb{E}_{\substack{x\sim \mu_n \\ x'\sim \mu_n}}\left[\left(L(s_{t}(x'))\|\phi(x)\|+L(s_{t'}(x))\left\|\phi(x')\right\|\right) \left\|\phi(x)\right\|\right]\\
&\le |t-t'|\left(\mathbb{E}_{x'\sim \mu_n} \left[L(s_{t}(x'))\right] \mathbb{E}_{x\sim \mu_n} \left[\|\phi(x)\|^2\right]+\mathbb{E}_{x\sim \mu_n} \left[L(s_{t'}(x))\|\phi(x)\|\right] \mathbb{E}_{x'\sim \mu_n} \left[\|\phi(x')\|\right]\right)\\
&\le |t-t'|\left(\mathbb{E}_{x\sim \mu_n} \left[L(s_{t}(x))\right] \mathbb{E}_{x\sim \mu_n} \left[\|\phi(x)\|^2\right]+\mathbb{E}_{x\sim \mu_n} \left[L(s_{t'}(x))^2\right]^{\frac12} \mathbb{E}_{x\sim \mu_n} \left[\|\phi(x)\|^2\right]^{\frac12} \mathbb{E}_{x'\sim \mu_n} \left[\|\phi(x')\|^2\right]^{\frac12}\right)\\
&\le \gamma^2|t-t'|\left(\|L\|_{L^1(\rho_{t})}+\|L\|_{L^2(\rho_{t'})}\right) \mathbb{E}_{x\sim \mu_n} \left[\|\nabla_{W_2}\cF(\mu_n)(x)\|^2\right] \text{ since $\phi(x)=-\gamma\nabla_{W_2}\cF(\mu_n)(x)$.}
\end{align*}
Hence, for $t'=0$, we have
$$\left|\dcF\left(\rho_{t}\right)-\dcF\left(\rho_{0}\right)\right| \leq t \gamma^{2} \left(\|L\|_{L^1(\rho_{t})}+\|L\|_{L^2(\mu_n)}\right) \mathbb{E}_{x\sim \mu_n} \left[\|\nabla_{W_2}\cF(\mu_n)(x)\|^2\right].
$$
Then, since 
$\rho_t=((1-t)I+t(I+\phi))_{\#}\mu_n$, by convexity of $L$, we can write :
\begin{align*}
    \|L\|_{L^1(\rho_t)}&=\int |L((1-t)x+t(x+\phi(x))|\d\mu_n(x)\\
&    \le (1-t)\|L\|_{L^1(\mu_n)} +t\|L\|_{L^1(\mu_{n+1})}\le (1-t)\|L\|_{L^2(\mu_n)} +t\|L\|_{L^2(\mu_{n+1})} \le M.
\end{align*}
We can thus upper bound the second term in the r.h.s. of \eqref{eq:develop}, 
\begin{equation*}
    \int_{0}^{1} |\dcF\left(\rho_{t}\right)-\dcF\left(\rho_{0}\right)| d t\le \int_{0}^{1} \left(t 2M \gamma^2 \mathbb{E}_{x\sim \mu_n} \left[\|\nabla_{W_2}\cF(\mu_n)(x)\|^2\right]\right) \d t = \gamma^2 M \mathbb{E}_{x\sim \mu_n} \left[\|\nabla_{W_2}\cF(\mu_n)(x)\|^2\right].
\end{equation*}
Finally, since $\gamma M \le 1$, \eqref{eq:develop} leads to
\begin{equation*}
    \cF(\mu_{n+1})-\cF(\mu_n)\le -\gamma(1-\gamma M)\mathbb{E}_{x\sim \mu_n} \left[\|\nabla_{W_2}\cF(\mu_n)(x)\|^2\right]\le 0.\qedhere
\end{equation*}

\section{Additional results}

\subsection{On the kernel integral operator}
\begin{lemma}\label{lem:kernel_operator_sp}
	For any $f,g \in L^2(\pi)$, we have that
	\begin{equation*}
	\ps{S_{\pi, \kpi}f,S_{\pi, \kpi}g}_{\Hkpi}
	=\iint f(x)^\top g(y)\kpi(x,y)\d\pi(x)\d\pi(y).
	\end{equation*}
\end{lemma}

\begin{proof}
	By using the reproducing property, we deduce that
	\begin{multline*}
	\ps{S_{\pi, \kpi}f,S_{\pi, \kpi}g}_{\Hkpi^d}
	=\ps{\int \kpi(x,\cdot)f(x)d\pi(x),\int \kpi(y,\cdot)g(y)d\pi(y)}_{\Hkpi^d}\\
	 =\sum_{i=1}^d \ps{\int \kpi(x,\cdot)f_i(x)d\pi(x),\int \kpi(y,\cdot)g_i(y)d\pi(y)}_{\Hkpi}\\
	 =\sum_{i=1}^d \iint f_i(x)\kpi(x,y)g_i(y)\d\pi(x)\d\pi(y)
	=\iint f(x)^\top g(y)\kpi(x,y)\d\pi(x)\d\pi(y).\qedhere
	\end{multline*}
\end{proof}
\subsection{On the differentiation of the squared KSD}
For \Cref{lem:derivative_ksd} below, our computations are similar to the ones in \citet[Lemma 22 and 23]{arbel2019maximum} with some terms getting simpler owing to the Stein's property of the Stein kernel, but under a weaker assumption than a uniform Lipschitz constant for the kernel (see the discussion in \Cref{sec:descent_lemma}).

\begin{lemma}\label{lem:derivative_ksd}
Let $q \in \cP_2(\X)$ and $\phi\in  C^{1}_c(\X)$. Consider the path $\rho_t$ from  $q$ to $(I+\nabla\phi)_{\#}q$ given by: $\rho_t=  (I+t\nabla\phi)_{\#}q$, for all $t\in [0,1]$. Suppose \Cref{ass:lipschitz,ass:square_der_int} hold. 
Then, $\mathcal{F}(\rho_t)$ is twice differentiable in $t$ with
	\begin{align*}
		\dot{\cF}(\rho_t)=&\mathbb{E}_{\substack{(x)\sim q\\(x')\sim q}}\left[ \nabla_{1} \kpi(x+t\nabla\phi(x),x'+t\nabla\phi(x'))^{\top} \nabla \phi(x)\right],\\
\ddot{\cF}(\rho_t)=&\mathbb{E}_{\substack{(x)\sim q\\(x')\sim q}}\left[ \nabla \phi(x')^\top\nabla_1 \nabla_2 \kpi(x+t\nabla\phi(x),x'+t\nabla\phi(x'))\nabla \phi(x)\right]\\
&+ \mathbb{E}_{\substack{(x)\sim q^*\\(x')\sim q^*}}\left[ \nabla\phi(x)^\top H_1 \kpi(x+t\nabla\phi(x),x'+t\nabla\phi(x'))\nabla \phi(x)\right].
	\end{align*}
\end{lemma}
\begin{proof}
The function $f: t\mapsto \kpi(x+t\nabla \phi(x),x'+t\nabla\phi(x'))$ is differentiable for all $x,x' \in \X$, and its time derivative is :
\begin{align}\label{eq:f_dot}
        \dot{f}&=
        \nabla_{1} \kpi(x+t\nabla\phi(x),x'+t\nabla\phi(x'))^{\top}\nabla\phi(x) 
        + \nabla_{2} \kpi(x+t\nabla\phi(x),x'+t\nabla\phi(x'))^{\top}\nabla\phi(x')\nonumber\\
            &=\nabla_{1} \kpi(x+t\nabla\phi(x),x'+t\nabla\phi(x'))^{\top}\nabla\phi(x) + \nabla_{1} \kpi(x'+t\nabla\phi(x'),x+t\nabla\phi(x))^{\top}\nabla\phi(x')
\end{align}
using the symmetry of $\kpi$. 
The two terms on the r.h.s. of the former equation are symmetric w.r.t.\ $x$ and $x'$, so we will focus on the first one. 
By the Cauchy-Schwartz inequality and \Cref{ass:lipschitz}, 
\begin{align*}
|\nabla_1 k_{\pi}(x+t\nabla \phi(x), x'+t\nabla\phi(x'))^{\top}\nabla \phi(x) |&\leq \|\nabla_1 k_{\pi}(x+t\nabla \phi(x), x'+t\nabla\phi(x'))\| \|\nabla \phi(x)\|\\ 
&\le \left(L(x+t\nabla\phi(x)) \|x' + t\nabla\phi(x')\| + \|\nabla_1\kpi(x+\nabla\phi(x),0)\| \right)\|\nabla\phi(x)\|
\end{align*}
The r.h.s. of the above inequality is integrable in $x$ because $\nabla \phi$ is compactly supported, $L$ is continuous, and because $x'\mapsto \|x'+t\nabla \phi(x')\|$ is integrable since $q\in \cP_2(\X)$.
Therefore, by the differentiation lemma \citep[Theorem 6.28]{klenke2013probability}, $\cF(\rho_t)$ is differentiable and  $\dot{\cF}(\rho_t)=\E_{(x)\sim q,(x')\sim q}[\dot{f}]$, i.e.\ 
\begin{align*}
	\dot{\cF}(\rho_t)&=\mathbb{E}_{\substack{(x)\sim q\\(x')\sim q}}\left[ \nabla_{1} \kpi(x+t\nabla\phi(x),x'+t\nabla\phi(x'))^{\top}\nabla \phi(x)\right].
	\end{align*}
Now define the function $g:t\mapsto \nabla_{1} \kpi(x+t\nabla\phi(x),x'+t\nabla \phi(x'))^{\top} \nabla\phi(x)$. Its time derivative writes as
\begin{align*}
    \dot{g}&=\nabla\phi(x)^\top \nabla_2\nabla_1 \kpi(x+t\nabla\phi(x),x'+t\nabla\phi(x'))\nabla\phi(x') 
	+\nabla \phi(x)^\top H_1 \kpi(x+t\nabla\phi(x),x'+t\nabla\phi(x') ) \nabla\phi(x).
\end{align*}
The first term on the r.h.s. is integrable in $(x,x')$ because it is compactly supported and continuous.
We now deal with the second term. Under \Cref{ass:square_der_int}, $H_1 \kpi(x,y)$ is dominated by a $q$-integrable function. 
Then,
\begin{align*}
    |\nabla\phi(x)^\top H_1 \kpi(x+t\nabla\phi(x),x'+t\nabla\phi(x') ) \nabla\phi(x)|&\le \|H_1 \kpi(x+t\nabla\phi(x),x'+t\nabla\phi(x') )\|_{op}\| \nabla\phi(x)\|^2
    \end{align*}
The r.h.s. of the above inequality is integrable because $(x,y)\mapsto H_1 \kpi(x,y)$ is integrable and $\nabla \phi$ is bounded.
\end{proof}


\section{Additional experiments}\label{sec:more_experiments}

\textbf{Implementation details.} The code for KSD Descent is written in Python using Pytorch~\citep{paszke2019pytorch}. We use Matplotlib~\citep{matplotlib} for figures, Scipy~\citep{scipy} for the L-BFGS implementation, as well as Numpy~\citep{harris2020array}. It is available at \url{https://github.com/pierreablin/ksddescent}.

\textbf{Different initializations and variance for Gaussian mixtures.} To discuss in greater detail the results of our second toy example presented in \Cref{sec:experiments_toy} and illustrated on \Cref{fig:mog_problem}, we performed some more experiments for several choices of initialization and variance, for a mixture of two Gaussians with equal variance (see \Cref{fig:more_mog}). We also investigated the support of the stationary points of KSD Descent for a mixture of three Gaussians with different variances (see \Cref{fig:skewed_mog}). As discussed in \Cref{sec:proof_stable_support}, the support is not necessarily a smooth submanifold or an axis of symmetry.

\textbf{A comparison between KSD Descent and Stein points.} We finally compare KSD Descent and Stein points~\cite{chen2018stein}. We choose a low dimensional problem, since the approach of Stein points cannot scale to large dimensions. We use the classical $2$-D ``banana'' density. \Cref{fig:banana} shows the behavior of the two algorithms. Interestingly, KSD Descent succeeds and does not fall into spurious local minima, even tough the density is not log-concave. We posit that this happens because here the potential $\log(\pi)$ does not have saddle points.

\begin{figure*}
   \centering
\begin{tabularx}{\textwidth}{cCCC}
Variance &  Initialization exactly on the symmetry axis & Initialization close to the symmetry axis & Gaussian i.i.d.\ initialization\\
 $0.1$ &
	\includegraphics[width=.25\columnwidth]{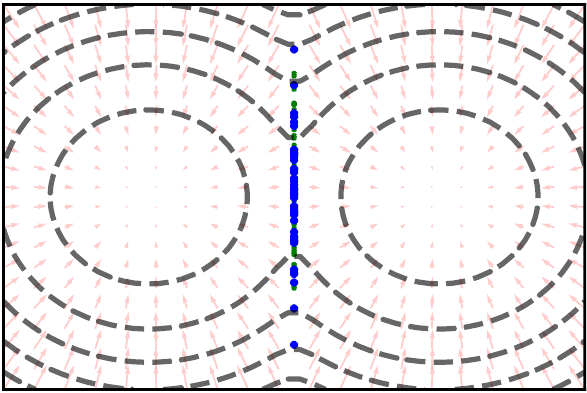}&
	\includegraphics[width=.25\columnwidth]{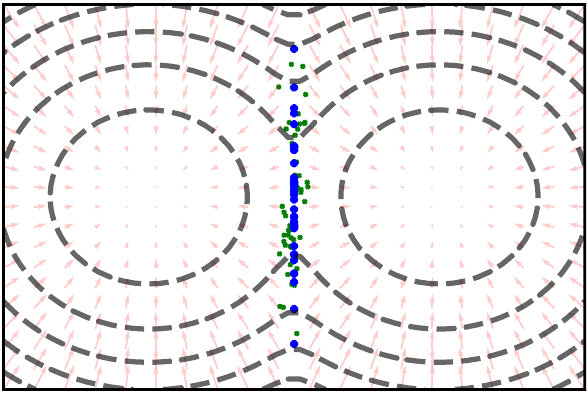}&	\includegraphics[width=.25\columnwidth]{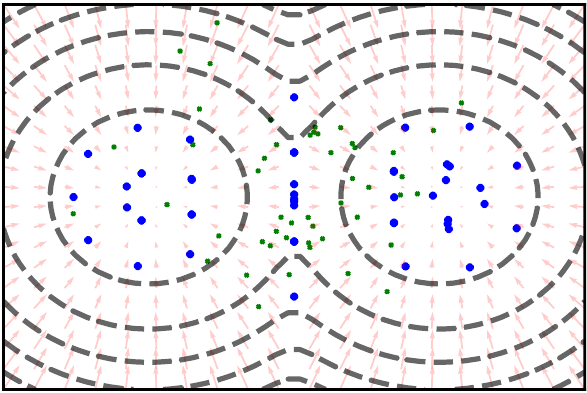} \\
$0.3$ &
	\includegraphics[width=.25\columnwidth]{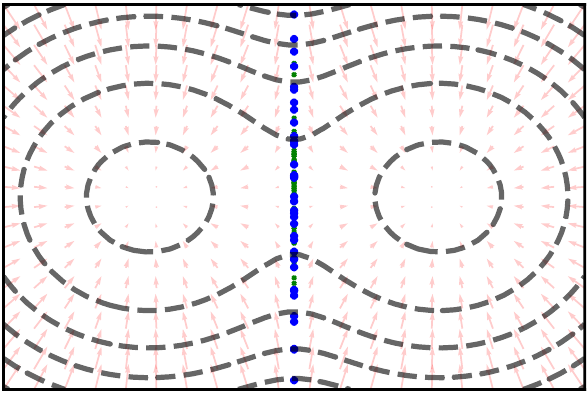}&
	\includegraphics[width=.25\columnwidth]{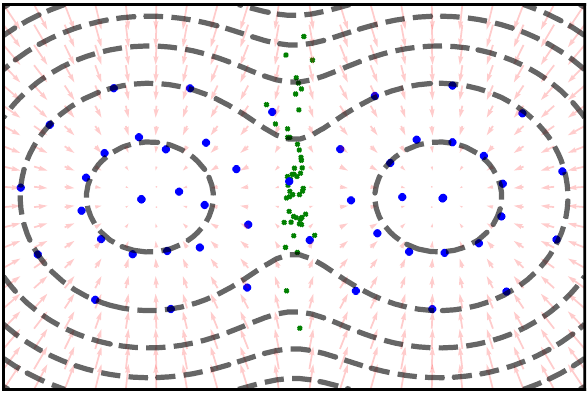}&
	\includegraphics[width=.25\columnwidth]{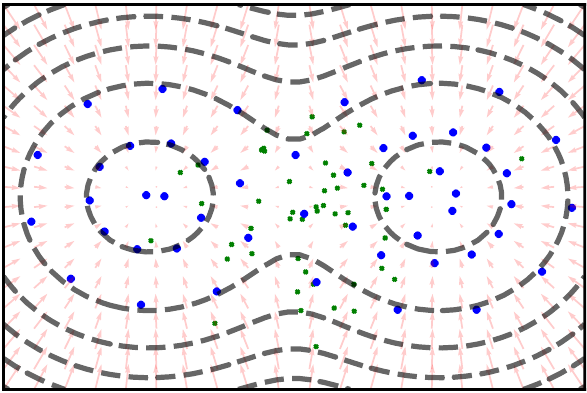} \\
$2$ &
	\includegraphics[width=.25\columnwidth]{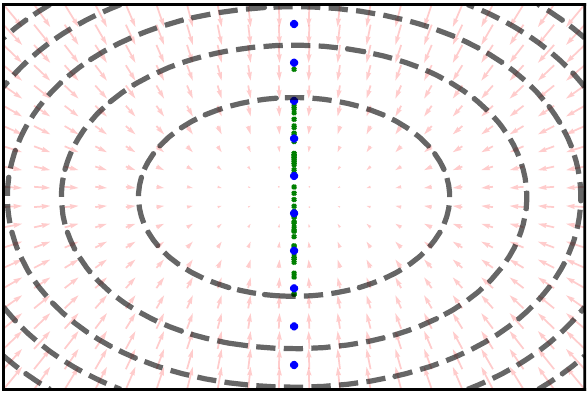}&
	\includegraphics[width=.25\columnwidth]{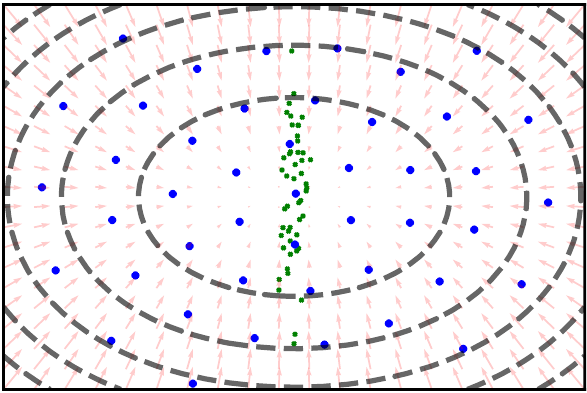}&
	\includegraphics[width=.25\columnwidth]{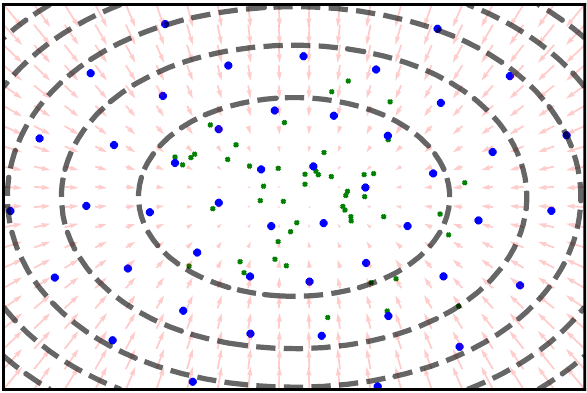} 
\end{tabularx}
    \caption{Results of KSD Descent for a mixture of two gaussians, depending on their variance and on the initialization of the algorithm. The green crosses indicate the initial particle positions, while the blue ones are the final positions.}
    \label{fig:more_mog} 
\end{figure*}

\begin{figure}
	\centering
	\includegraphics[width=.50\columnwidth]{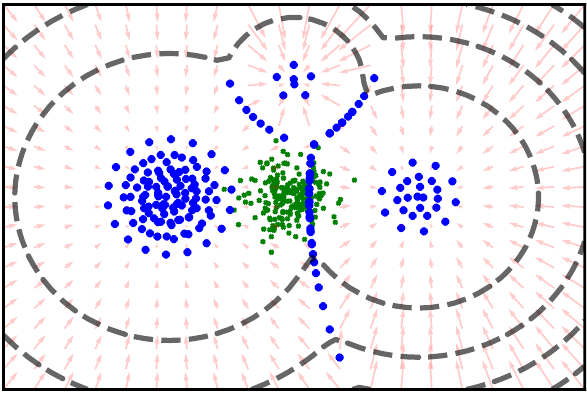}
	\caption{Using KSD Descent to sample from an unbalanced mixture of Gaussians. Some particles get stuck in spurious zones, which are not a straight line nor a manifold. The green crosses indicate the initial particle positions, while the blue ones are the final positions.}
	\label{fig:skewed_mog}
\end{figure}

\begin{figure}
	\centering
	\includegraphics[width=.4\columnwidth]{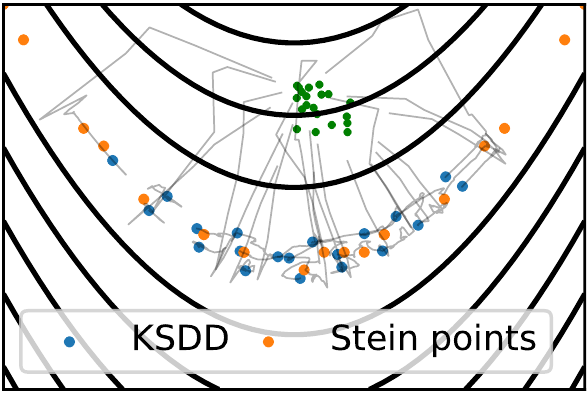}
	\caption{Comparison of KSD Descent and Stein points on the ``banana'' dataset. Green points are the initial points for KSD Descent. Both methods work successfully here, even though it is not a log-concave distribution. We posit that KSD Descent succeeds because there is no saddle point in the potential.}
	\label{fig:banana}
\end{figure}

\end{document}